\documentclass[10pt,twocolumn,letterpaper]{article}

\usepackage{lineno}
\usepackage[pagenumbers]{cvpr} 
\usepackage{amsthm,amssymb}
\usepackage{graphicx}
\usepackage[ruled,linesnumbered]{algorithm2e}
\usepackage{mathtools}
\usepackage{algorithmic}
\usepackage{url}
\usepackage{tabu}
\usepackage{multirow}
\usepackage{makecell}
\usepackage{colortbl}
\usepackage[pagebackref,breaklinks,colorlinks,allcolors=cvprblue]{hyperref}
\usepackage{bm}
\usepackage{pifont}
\definecolor{cvprblue}{rgb}{0.21,0.49,0.74}
\definecolor{red}{RGB}{238, 136, 116}
\definecolor{green}{RGB}{110, 202, 205}
\definecolor{mygray}{gray}{0.4}
\definecolor{demphcolor}{RGB}{90,90,90}

\newcommand{\tablestyle}[2]{\setlength{\tabcolsep}{#1}\renewcommand{\arraystretch}{#2}\centering\footnotesize}
\newlength\savewidth\newcommand\shline{\noalign{\global\savewidth\arrayrulewidth
  \global\arrayrulewidth 1pt}\hline\noalign{\global\arrayrulewidth\savewidth}}

\newcommand{\cmark}{\color{mygray}\ding{51}}%
\newcommand{\demph}[1]{\textcolor{demphcolor}{#1}}
\def\paperID{4601} 
\def\confName{CVPR}
\def\confYear{2025}

\newtheorem{theorem}{Theorem}
\newtheorem{lemma}{Lemma}
\newtheorem{proposition}{Proposition}

\usepackage[titles]{tocloft}
\usepackage{minitoc}
\usepackage{etoc}
\usepackage[symbol]{footmisc}

\usepackage[titletoc,title]{appendix}

\title{RayFlow: Instance-Aware Diffusion Acceleration via Adaptive Flow Trajectories}

\makeatletter
\def\thanks#1{\protected@xdef\@thanks{\@thanks
        \protect\footnotetext{#1}}}
\makeatother
\author{
Huiyang Shao\quad Xin Xia$^*$ \quad
Yuhong Yang \quad Yuxi Ren \quad Xing Wang \quad
Xuefeng Xiao$^\dagger$ \\\\
\textbf{ByteDance Inc.} \thanks{
$^*$ Corresponding Author}\thanks{$^\dagger$ Project Leader}
}

\begin{document}
\maketitle
\begin{figure*}[!t]
  \centering
  \begin{subfigure}[b]{0.49\textwidth} %
    \includegraphics[width=\textwidth]{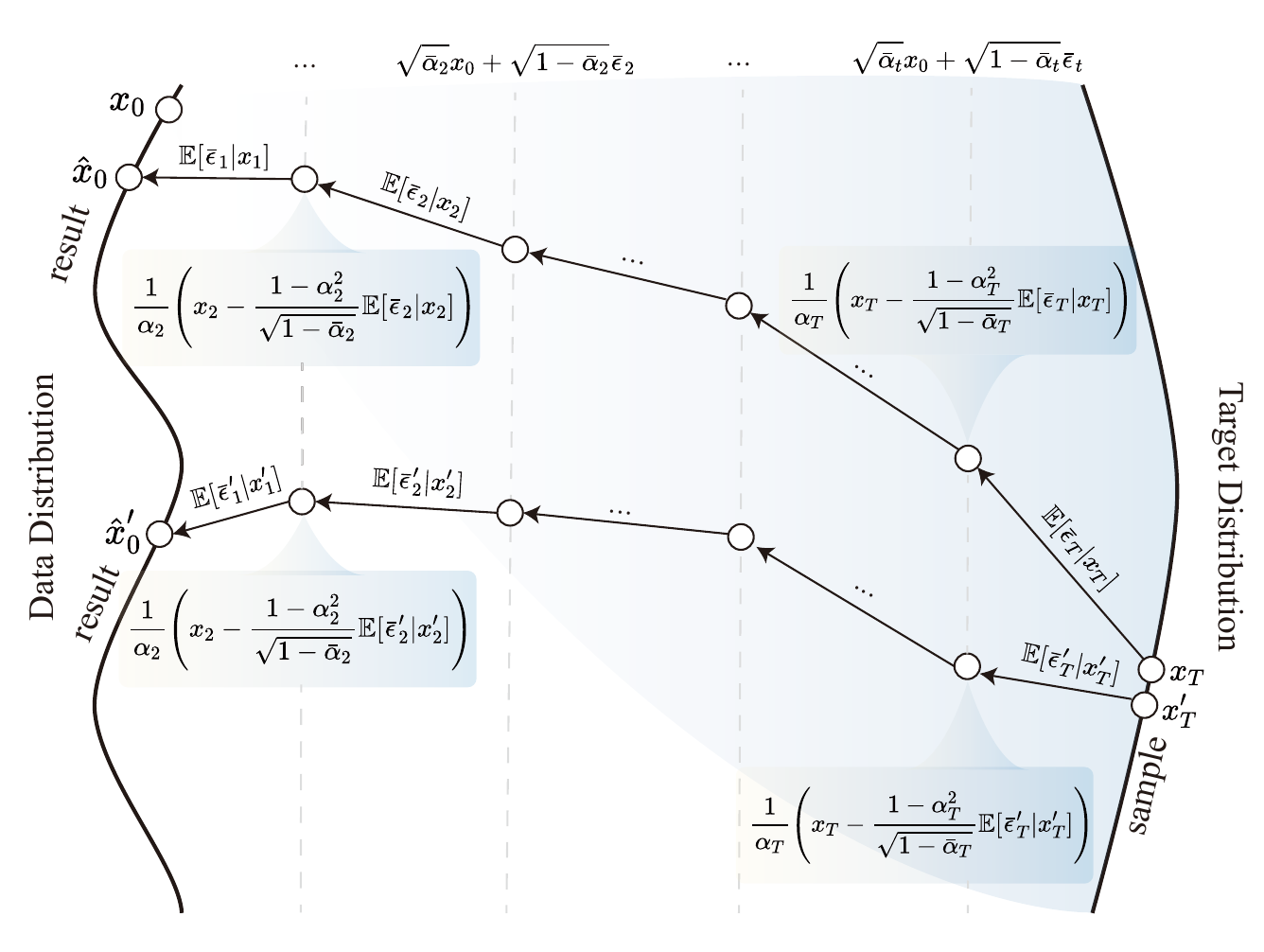}
    \caption{The forward and backward process of traditional diffusion.}
    \label{fig:traditional_diffusion}
  \end{subfigure}
  \hfill
  \begin{subfigure}[b]{0.49\textwidth} 
    \includegraphics[width=\textwidth]{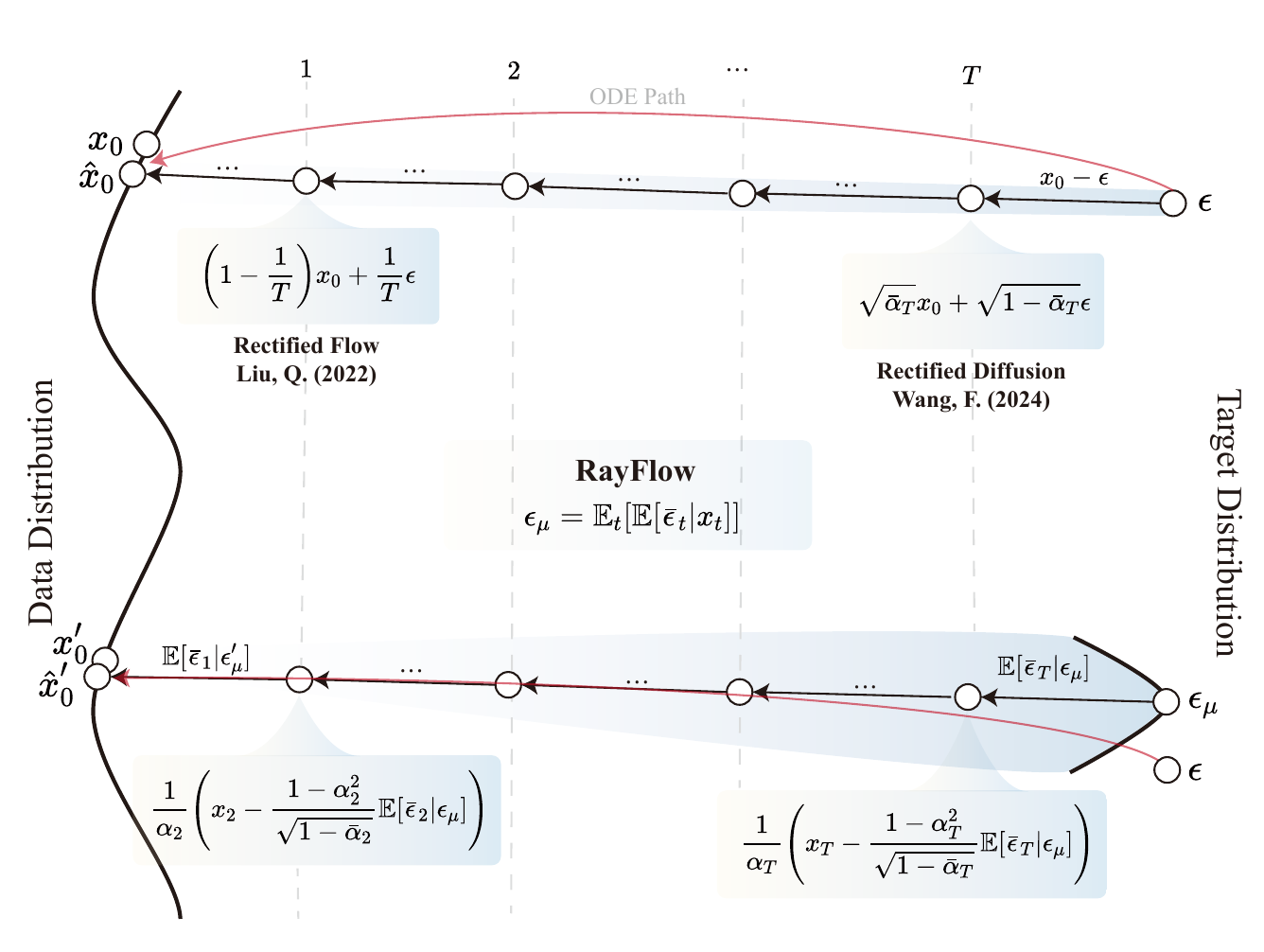}
    \caption{The forward and backward process of our diffusion method.}
    \label{fig:our_diffusion}
  \end{subfigure}
  \caption{The comparision of different diffusion process.}
  \label{fig:diffusion_compare}
\end{figure*}

\begin{abstract}
Diffusion models have achieved remarkable success across various domains. However, their slow generation speed remains a critical challenge. Existing acceleration methods, while aiming to reduce steps, often compromise sample quality, controllability, or introduce training complexities. Therefore, we propose \underline{\textit{RayFlow}}, a novel diffusion framework that addresses these limitations. Unlike previous methods, RayFlow guides each sample along a unique path towards an instance-specific target distribution. \textbf{This method minimizes sampling steps while preserving generation diversity and stability.} Furthermore, we introduce Time Sampler, an importance sampling technique to enhance training efficiency by focusing on crucial timesteps. Extensive experiments demonstrate RayFlow's superiority in generating high-quality images with improved speed, control, and training efficiency compared to existing acceleration techniques.
\end{abstract}
\section{Introduction}

Diffusion models have revolutionized generative AI \citep{sohl2015deep,song2019generative,ho2020ddpm,song2020score}, achieving impressive results in various domains, from text \citep{li2022diffusion, wu2023ar} and images \citep{xie2024addsr, shao2023building,zhao2024wavelet,zhao2024learning} to 3D models \citep{wang2024prolificdreamer,liu2023zero, cai2024digital}, audio \citep{liu2023audioldm,evans2024fast}, and restoration \citep{blattmann2023stable,videoworldsimulators2024}. However, their slow generation speed, often requiring dozens of steps per sample, remains a significant limitation.  Various distillation techniques attempt to accelerate this process, including normal \citep{luhman2021knowledge,zheng2023dfno}, adversarial \citep{wang2022diffusion-stylegan,sauer2023adversarial}, progressive \citep{salimans2022progressive}, and variational score distillation \citep{wang2024prolificdreamer,yin2024dmd,yin2024dmd2,luo2024diff-instruct,xie2024EMD,salimans2024multistepmoment}.  Current distillation methods, despite their objectives, are hindered by challenges including substantial computational overhead, complex training regimens, and limitations in terms of generation speed, sample quality, and effective guidance \citep{luo2023one, mei2024codi, xie2024em, nguyen2024swiftbrush, kang2024distillingdiffusionmodelsinto, zhou2024simple, zhang2024accelerating}.  This underscores the imperative for developing more generalized approaches to fully exploit the capabilities of diffusion models.

The traditional diffusion process, as shown in Fig.\ref{fig:traditional_diffusion}, forms the foundation for most acceleration and distillation methods. However, it's clear from the figure that this process suffers from several issues:
\begin{itemize}
    \item \textcolor{red}{{(1) Varying Expectations.}} The expectations in the backward process differ across timesteps. Achieving high-quality samples necessitates a greater number of sampling steps, making it impossible for accelerated sampling methods to avoid quality degradation.
    \item \textcolor{red}{(2) Overlapping Diffusion Paths.} Since all sample eventually converges to the same standard Gaussian distribution, the diffusion probability paths will overlap.  $\mathbb{E}[\bar{\epsilon}_t|x_t]$ may represent the intersection of multiple probability paths, leading to significant randomness in sampling outcomes and potentially substantial quality loss.
    \item \textcolor{red}{(3) Sampling Instability.} Even with closely positioned sampling points, the final generated results can differ significantly, introducing substantial instability.
\end{itemize}

Several studies try to address these limitations: \citep{liu2022flow} introduced rectified flow (RF) using linear ordinary differential equation (ODE) for straight-path sampling (Fig.\ref{fig:our_diffusion} top-left), while \citep{wang2024rectified} extended RF to first-order curved paths rectified diffusion (RD) for denoising diffusion probabilistic models (DDPM)-like models (Fig.\ref{fig:our_diffusion} top-right). These sample-noise matching approaches have also been applied to diffusion distillation by \citep{liu2023instaflow, yan2024perflow}.

However, existing sample-noise matching algorithms brings new drawbacks while attempting to address above challenges: 
\textcolor{red}{(1*) Path Inconsistency:} the gap between sample-noise matching and actual ODE sampling paths is too large, potentially leading to training difficulties and poor generalization; 
\textcolor{red}{(2*) 
Limited Diversity:} sampling probability paths are severely constrained, significantly reducing model generation diversity; \textcolor{red}{(3*) Theoretical Gap:} the method is highly intuitive but lacks fundamental theoretical derivation to prove its optimality for sampling stability.

To address all the three challenges in traditional diffusion and the drawbacks introduced by existing sample-noise matching methods, we propose \underline{\textit{RayFlow}} (shown in Fig.\ref{fig:our_diffusion} bottom).  \textcolor{green}{\textbf{Consistent Expectations and Path:}} For challenge \textbf{(1)} and \textbf{(1*)}, we leverage pre-trained models to calculate a unified noise expectation $\epsilon_\mu=\mathbb{E}_t[\mathbb{E}[\bar{\epsilon}_t]]$ across all timesteps, enabling efficient step compression without quality degradation. \textcolor{green}{\textbf{Individual Path Design:}} For challenge \textbf{(2)} and \textbf{(2*)}, instead of converging to a common Gaussian, each sample follows a unique diffusion path towards its specific target mean with reduced variance, minimizing path overlaps and sampling randomness.
\textcolor{green}{\textbf{Theoretical Guarantee:}} For challenge \textbf{(3)} and \textbf{(3*)}, we prove that our method maximizes the path probability between the starting point, target mean, and origin, ensuring optimal sampling stability and reliable reconstruction of the original data point.

Moreover, to improve training efficiency, we develop \textbf{Time Sampler}, an advanced importance sampling method that identifies crucial timesteps during training. By combining Stochastic Stein Discrepancies (SSD) with neural networks, Time Sampler approximates the optimal sampling timestep distribution to minimize variance of training loss estimator, thereby reducing computational redundancy and enhancing efficiency.

Our key contributions can be summarized as follows:
\begin{itemize}
    \item \textbf{RayFlow Framework:} We introduce an innovative diffusion framework with instance-independent target means. This approach offers enhanced control over the generative process, allowing for more efficient and precise sampling.
    \item \textbf{Time Sampler:} We develop an timestep importance sampling technique utilizing SSD. This method effectively identifies crucial timesteps during training, reducing computational redundancy and enhancing efficiency.
    \item \textbf{Efficient Algorithms:} We present practical algorithms for training and sampling with RayFlow, including a fast one-step sampling variant for quicker generation.
    \item \textbf{Theoretical Analysis:} We provide a thorough theoretical examination of RayFlow, detailing the derivation of path probabilities and the optimization of parameters to ensure maximal sampling stability.
\end{itemize}
Through extensive experiments, we demonstrate the effectiveness of RayFlow in generating high-quality images with improved efficiency and controllability compared to existing acceleration algorithm. Our work opens up new  approach for exploring and controlling diffusion processes.

\section{Preliminaries}

\subsection{Denoising Diffusion Probabilistic Models \cite{ho2020ddpm}}\label{sec:DDPM}
Consider a dataset of real samples $\bm{x}_0 \in \mathbb{R}^d$ from distribution $p(\bm{x}_0)$. DDPMs learn this distribution through iterative noising and denoising over $T$ discrete time steps, indexed by $t \in \{1,\dots,T\}$. Let $\mathcal{N}(\bm{\mu}, \bm{\Sigma})$ denote the Gaussian distribution with mean $\bm{\mu}$ and covariance $\bm{\Sigma}$.

\textbf{Forward Process.} The forward diffusion process gradually adds Gaussian noise through a Markov chain:
\begin{equation}
p(\bm{x}_t|\bm{x}_{t-1})=\mathcal{N}(\alpha_t\bm{x}_{t-1},\beta_t\bm{I})
\end{equation}
where $\alpha_t\in[0,1]$ is the draft term and $\beta_t\in[0,1]$ is noise variance.  After $T$ steps, $p(\bm{x}_T) = \mathcal{N}(\bm{0}, \bm{I})$. With setting $\alpha_t^2 + \beta_t^2 = 1$, the process simplifies to:
\begin{equation}\label{eq:diffusion_forward_prod}
    p(\bm{x}_t|\bm{x}_0) = \mathcal{N}(\bar{\alpha}_t\bm{x}_0, 1-\bar{\alpha}_t\bm{I})
\end{equation}
where $\bar{\alpha}_t = \prod_{s=1}^t \alpha_s$ is the cumulative product of scaling factors up to step $t$.

\textbf{Backward Process.}
The reverse process reconstructs $\bm{x}_0$ from $\bm{x}_T$ by iteratively denoising:
\begin{equation}
\begin{aligned}
p(\bm{x}_{t-1}|\bm{x}_t,\bm{x}_0) = \mathcal{N}(\tilde{\bm{\mu}}(\bm{x}_t,\bm{x}_0), \hat{\beta}_t\bm{I})
\end{aligned}
\end{equation}
where $\tilde{\bm{\mu}}$ are parameters which need to learn, $\hat{\beta}_t = (1-\bar{\alpha}_{t-1})/(1-\bar{\alpha}_t)\beta_t$.

\subsection{Simulation-Free Flow Matching}
Flow matching (FM) \citep{esser2024sd3, liu2022flow,liu2209rectified,lipman2022flow} consider generative models that learn a mapping between a noise distribution $p(\bm{x}_T)$ and a data distribution $p(\bm{x}_0)$ via an ODE:
\begin{equation}
  \label{eq:ode}
  d\bm{x}_t = \bm{v}_{\bm{\theta}}(\bm{x}_t, t)dt,
\end{equation}
where $\bm{v}_{\bm{\theta}}(\bm{y}_t, t)$ represents a velocity field parameterized by a neural network with weights $\bm{\theta}$.  Let $\bm{u}(\bm{x}_t)$ be the marginal velocity field that generates the marginal probability $p(\bm{x}_t)$.
\begin{equation*}
\psi_t(\cdot |\bm{\epsilon}) = a_{t} \bm{x}_{0} + b_{t} \bm{\epsilon}, \ \
\bm{u}(\bm{x}_t |\bm{\epsilon}) = \psi_{t}^{\prime}\left(\psi_{t}^{-1}(\bm{x}_t |\bm{\epsilon}) |\bm{\epsilon}\right),
\end{equation*}
where $a_t, b_t\in\mathbb{R}$ are coefficient. The FM objective, which aims to directly regress marginal velocity field ($\bm{u}(\bm{x}_t)$ has the same gradient with $\bm{u}(\bm{x}_t|\bm{\epsilon})$ \cite{lipman2022flow}):
\begin{align}\label{Eq.diffusion_train}
   \mathcal{L}_{CFM} =  \mathbb{E}_{t, p(\bm{x}_t|\bm{\epsilon}),p(\bm{\epsilon})} \left[ || \bm{v}_{\bm{\theta}}(\bm{x}_t, t) - \bm{u}(\bm{x}_t|\bm{\epsilon}) ||_2^2 \right].
\end{align}

According to \citep{esser2024sd3}'s setting, we have $\bm{u}_{t}(\bm{x}_{t}|\bm{\epsilon}) = \frac{a_{t}^{\prime}}{a_{t}} \bm{x}_{t} - \frac{b_{t}}{2} \lambda_{t}^{\prime} \epsilon$, where the signal-to-noise ratio $\lambda_t = \log(a_t^2/b_t^2)$, with $\lambda_t' = 2 (\alpha_t'/\alpha_t - b_t'/b)$. Denote $\bm{\epsilon}_{\bm{\theta}}(\bm{x}_t, t) = \frac{-2}{\lambda_t' b_t} (\bm{v}_{\bm{\theta}} - \frac{a_t'}{a_t} \bm{x}_t)$, then we can rewrite Eq.\eqref{Eq.diffusion_train} to:
\begin{equation}\label{eq:fm_diffusion_train}
  \mathcal{L}(\bm{x}_0) = -\frac{1}{2} \mathbb{E}_{t, \bm{\epsilon}}
  \left[ -\frac{1}{2} \lambda_t'^2 b_t^2 \Vert \bm{\epsilon}_{\bm{\theta}}(\bm{x}_t, t) - \bm{\epsilon} \Vert^2 \right]
\end{equation}
Eq.\eqref{eq:fm_diffusion_train} shows the connection between diffusion and FM, demonstrating that their training processes differ only by different weighting factor.

\subsection{Flow Trajectories}
This section examines several trajectory variants.





\textbf{Variance Preserving (VP) \citep{ho2020ddpm}.}  As described in Section~\ref{sec:DDPM}, sets $\alpha_t^2+\beta_t^2=1$ and defines $\beta_t$ linearly between $\beta_0$ and $\beta_{T}$.  VP uses a square-root interpolation of $\sqrt{\bar{\alpha}_t}$, with forward process $\psi_t(\cdot|\bm{\epsilon})=\bm{x}_t=\sqrt{\bar{\alpha}_t}\bm{x}_0+\sqrt{1-\bar{\alpha}_t}\bm{\epsilon}$ and marginal velocity field $\bm{u}(\bm{x}_t|\bm{\epsilon})=\frac{\alpha_t'}{1-\alpha_t^2}(\alpha_t\bm{x}_0-\bm{\epsilon})$

\textbf{Rectified Flow \citep{liu2022flow}.}  RF defines a linear forward process: $\psi_t(\cdot|\bm{\epsilon})=\bm{x}_t=(1-t)\bm{x}_0+t\bm{\epsilon}$, and the marginal velocity field $\bm{u}(\bm{x}_t|\bm{\epsilon})=\frac{\bm{\epsilon}-\bm{x}_0}{1-t}$


\subsection{Importance Sampling}
In variational inference, we frequently seek to estimate expected values of the form \(\bm{\mu} = \mathbb{E}_{\bm{x} \sim p}[\xi(\bm{x})]\), where \(p\) is a probability distribution over \(\mathcal{D} \subseteq \mathbb{R}^d\) and \(\xi: \mathcal{D} \to \mathbb{R}\) is an integrable function. The classical Monte Carlo estimate of \(\bm{\mu}\) is \(\hat{\bm{\mu}}_n = \frac{1}{n} \sum_{i=1}^{n} \xi(\bm{x}_i)\), with \(\bm{x}_i \sim p\) being i.i.d. samples.

However, when \(\xi\) is non-zero primarily in regions where \(p\) is small, standard Monte Carlo sampling becomes inefficient. Importance Sampling addresses this by sampling from an alternative distribution \(q\), chosen to reduce the variance of the estimator.

\textbf{Definition of Importance Sampling}. Let $q\colon \mathcal{D} \to \mathbb{R}$ be a probability density function such that $q(\bm{x}) > 0$ for all $\bm{x} \in \mathcal{D}$ where $p(\bm{x}) \xi(\bm{x}) \neq 0$. We can express the expected value $\bm{\mu}$ as:
\begin{equation}
\bm{\mu} = \int_{\mathcal{D}} \xi(\bm{x}) \frac{p(\bm{x})}{q(\bm{x})} q({x})  \mathrm{d} {\bm{x}} = \mathbb{E}_{\bm{x} \sim q}\left[\xi(\bm{x}) \frac{p(\bm{x})}{q(\bm{x})}\right].
\end{equation}

\textbf{Importance Sampling Estimator}. Given i.i.d. samples $\{\bm{x}_i\}_{i=1}^{n}$ from the importance distribution $q$, the importance sampling estimator $\hat{\bm{\mu}}_q$ is defined as:
\begin{equation}
\hat{\bm{\mu}}_q = \frac{1}{n} \sum_{i=1}^n \xi({x}_i) \frac{p(\bm{x}_i)}{q(\bm{x}_i)}.
\end{equation}
This estimator is unbiased, meaning that its expectation is equal to $\bm{\mu}$. Moreover, we can get variance of Importance Sampling Estimator. Then the variance of the importance sampling estimator is given by $\mathbb{V}_q[\hat{\bm{\mu}}_q] = \frac{1}{n} \sigma_q^2$, where
\begin{equation}
    \sigma_q^2 = \mathbb{E}_{\bm{x} \sim q}\left[\xi^2(\bm{x}) \left(\frac{p(\bm{x})}{q(\bm{x})}\right)^2\right] - \bm{\mu}^2.
\end{equation}

\section{Proposed Method}

This section introduces our proposed method for RayFlow, which involves a novel approach to diffusion processes in generative modeling. We present the theoretical foundations, the forward and backward processes, and the path probability. We also provide the optimal parameters for maximizing the path probability and outline the training and sampling algorithms.

\begin{figure*}[htbp] 
  \centering
  \includegraphics[width=\textwidth]{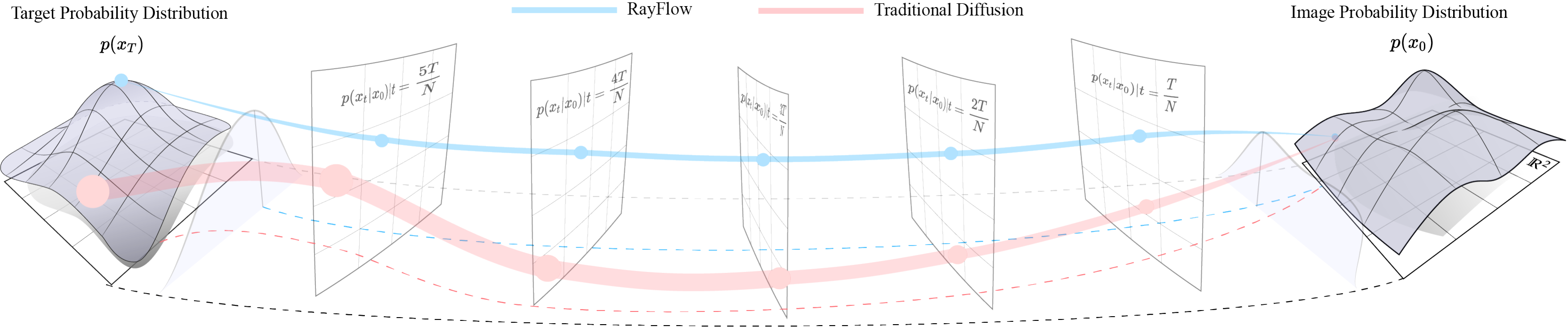} 
  \caption{RayFlow and importance time sampling. Time Sampler can find the key timesteps (five coordinates) of flow matching.} 
  \label{fig:wide} 
\end{figure*}

\subsection{RayFlow}
We introduce RayFlow, a novel framework that transforms data through a precise trajectory. The core of our method is a markov chain that construct a flow that sample $\bm{x}_0$ between target distribution $p(\bm{x}_T) = \mathcal{N}(\bm{\epsilon}_\mu, \sigma^2\bm{I})$, where $\bm{\epsilon}_\mu$ represents a pretrained mean vector and $\sigma$ is the standard deviation. This process is shown in following proposition (proof in Appendix.\ref{sec:RayFlow_trajectory}).
\begin{proposition}\label{prop:RrayFlow}
Given data $\bm{x}_0$, pretrained mean $\bm{\epsilon}_\mu \in \mathbb{R}^d$ and variance $\sigma \in \mathbb{R}_{>0}$, and target diffusion $p(\bm{x}_T) = \mathcal{N}(\bm{\epsilon}_\mu, \sigma^2\bm{I})$, we can describe the diffusion process with the following Markov chain.

\textbf{Flow Trajectories.} Define probability flow path.
\begin{equation}
    \psi_t(\cdot|\bm{\epsilon})=\sqrt{\bar{\alpha}_t}\bm{x}_0 + (1-\sqrt{\alpha}_t)\bm{\epsilon}_\mu + \sqrt{1-\bar{\alpha}_t}\bm{\epsilon}
\end{equation}

\textbf{Forward Process.} Add noise to image data.
\begin{equation*}
p(\bm{x}_t|\bm{x}_{t-1},\bm{\epsilon}_\mu)=\mathcal{N}\left(\alpha_t\bm{x}_{t-1}+(1-\alpha_t)\bm{\epsilon}_\mu, \beta_t^2\sigma^2\bm{I}\right)
\end{equation*}

\textbf{Backward Process.} Denoise from image data.
\begin{equation}
\begin{aligned}
&p(\bm{x}_{t-1}|\bm{x}_{t},\bm{\epsilon}_\mu)=\mathcal{N}\left(\frac{1}{\alpha_t} \bm{x}_t-\frac{1-\alpha_t}{\alpha_t}\bm{\epsilon}_\mu, \tilde{\beta}_t\sigma^2\bm{I}\right)
\end{aligned}  
\end{equation}
\begin{equation}
\tilde{\beta}_t=\left(\frac{(1-\alpha_t^2)(1-\bar{\alpha}_{t-1})}{1-\bar{\alpha}_{t}}\right)
\end{equation}  

\end{proposition}

\textbf{Path Probability and Optimization}. A key theoretical result of our framework is the characterization of optimal parameters that maximize the path probability - the likelihood of successfully transforming an image through the forward and backward processes. The optimal parameters are defined by following proposition (proof in Appendix.\ref{sec:probability_path}):
\begin{proposition}\label{prop:probability_path}
For any general diffusion defined in Prop.\ref{prop:RrayFlow}, the path probability (the probability of starting from $\hat{\bm{x}}_0$, forward to $\hat{\bm{\epsilon}}_\mu$, and backward to $\hat{\bm{x}}_0$) is given by:
\begin{equation}
\begin{aligned}
&\qquad p(\hat{\bm{x}}_0\rightarrow\hat{\bm{\epsilon}}_\mu\rightarrow \hat{\bm{x}}_0) = \\
&\underbrace{p(\bm{x}_{T}=\hat{\bm{\epsilon}}_\mu|\bm{x}_0=\hat{\bm{x}}_0)}_{\text{Forward Path Prob.}}\cdot \underbrace{p(\bm{x}_0=\hat{\bm{x}}_0|\bm{x}_{T}=\hat{\bm{\epsilon}}_\mu)}_{\text{Backward Path Prob.}}
\end{aligned}
\end{equation}
where the backward path probability is defined by:
\begin{equation}
\begin{aligned}
&\qquad p(\bm{x}_{0}|\bm{x}_T=\hat{\bm{\epsilon}}_\mu)=
\\ 
&\mathcal{N}\left(\frac{\hat{\bm{\epsilon}}_\mu}{\sqrt{\bar{\alpha}_T}}  + \sum_{s=1}^{T}\frac{(e)_s+(c)_s}{\sqrt{\bar{\alpha}_{s-1}/\bar{\alpha}_{t-1}}}, \sum_{s=1}^T\frac{\tilde{\beta}_s\bar{\alpha}_t}{\bar{\alpha}_s}\sigma^2\bm{I}\right)
\end{aligned}
\end{equation}
and $(e)_s$ and $(c)_s$ are defined by:
\begin{equation}
(e)_s = -\frac{\sqrt{\bar{\alpha}_{s-1}}(1-\alpha_s^2)}{(1-\bar{\alpha}_{s})\sqrt{\bar{\alpha}_s}}\mathbb{E}[\bar{\bm{\epsilon}}_s]
\end{equation}
\begin{equation*}
(c)_s = \frac{1-\alpha_s-\bar{\alpha}_{s}+\alpha_s\bar{\alpha}_{s-1}-\sqrt{\bar{\alpha}_{s-1}}+\sqrt{\bar{\alpha}_{s}}\alpha_s}{1-\bar{\alpha}_{s}}\bm{\epsilon}_\mu
\end{equation*}
$\mathbb{E}[\bar{\bm{\epsilon}}_s]$ denotes the noise mean during forward at timestep $s$.
\end{proposition}
By maximizing the probabilistic path, we can minimize the instability of sampling. But how to choose the parameters that make the probabilistic path maximized is a important issue, so we propose the following theorem (proof in Appendix.\ref{sec:optimal_probability_path}).

\begin{theorem}\label{thm:optimal_parameters}
Let $\bm{S}_{\bar{\bm{\epsilon}}}=\{\bar{\bm{\epsilon}}_t\}_{t=1}^T$ be the noise added in the forward process, $\bm{\epsilon}_\mu$, $\sigma$ be the parameters of the target distribution $\mathcal{N}(\bm{\epsilon}_\mu, \sigma^2\bm{I})$. For sample $\hat{\bm{x}}_0$, we can obtain the optimal parameters that maximize the path probability.
\begin{equation}
\begin{aligned}
\underset{\bm{S}_{\bar{\bm{\epsilon}}},\hat{\bm{\epsilon}}_\mu,\bm{\epsilon}_\mu,\sigma}{\arg\max}\ \ p(\bm{x}_{0}=\hat{\bm{x}}_0|\bm{x}_T=\hat{\bm{\epsilon}}_\mu)\prod_{t=1}^T p(\bar{\bm{\epsilon}}_t=\mathbb{E}[\bar{\bm{\epsilon}}_t])
\end{aligned}
\end{equation}
where the optimal parameters are defined by:
\begin{equation}
\begin{aligned}
    &\bm{S}_{\bar{\bm{\epsilon}}}^*=\{(1-\sqrt{\bar{\alpha}_t})\bm{\epsilon}_\mu\}_{t=1}^T, \ \sigma^* \rightarrow 0\\
    &\hat{\bm{\epsilon}}_\mu^* = \sqrt{\bar{\alpha}_T}\hat{\bm{x}}_0 + (1-\sqrt{\bar{\alpha}_T})\bm{\epsilon}_\mu,\ \bm{\epsilon}_\mu^*=\mathbb{E}_t[\mathbb{E}[\bar{\bm{\epsilon}_t}]]
\end{aligned}
\end{equation}
\end{theorem}
This theorem provides us with a way to find the corresponding parameters of the target distribution that leads to a optimal probability path, which ensures that the diffusion process is efficient and informative.

\subsection{Timestep Sampling}
Training diffusion models involves estimating the expectation in Eq.\eqref{Eq.diffusion_train}, which averages the loss over all timesteps t.  While a uniform sampling of t is  inefficient due to high variance and redundant computations.  To address this, we introduce Time Sampler, a novel importance sampling technique for efficient timestep selection during training. Ideally, we want to sample timesteps from a distribution that minimizes the variance of Eq.\eqref{Eq.diffusion_train}. This optimal sampling distribution is shown in following proposition (proof in Appendix.\ref{sec:optimal_sampling_distribution}).

\begin{proposition}\label{prop:optimal_sampling_distribution}
The optimal sampling distribution for Eq.\eqref{Eq.diffusion_train} with minimal variance is:
\begin{equation}
\begin{aligned}
q^*(t|\bm{x}_0,\bm{\epsilon}_\mu)\propto\xi_t(\bm{x}_0,\bm{\epsilon}_\mu)p(t),
\end{aligned}
\end{equation}
where $\xi_t(\bm{x}_0,\bm{\epsilon}_\mu)=\|\bm{\epsilon}_{\bm{\theta}}(\sqrt{\bar{\alpha}_t}\hat{\bm{x}_0}+(1-\sqrt{\bar{\alpha}_t})\bm{\epsilon}_\mu)-\bm{\epsilon}_\mu\|_2^2$.
which means for any probability  distribution $p$, we have
\begin{equation*}
\begin{aligned}
\mathbb{V}_{t\sim q^*(t),(\bm{x}_0,\bm{\epsilon}_\mu)}[\xi_t(\bm{x}_0,\bm{\epsilon}_\mu)]\leq\mathbb{V}_{t\sim p(t),(\bm{x}_0,\bm{\epsilon}_\mu)}[\xi_t(\bm{x}_0,\bm{\epsilon}_\mu)]
\end{aligned}
\end{equation*}

\end{proposition}

However, Estimating $q^*(t|\bm{x}_0,\bm{\epsilon}_\mu)$ presents two key challenges: \textbf{1) limited samples}: At each iteration, we can only access small batch of samples, making it difficult to accurately estimate the full distribution.
\textbf{2) data dependency}: The sampling distribution varies for different $\bm{x}_0$ and $\bm{\epsilon}_\mu$, requiring a flexible approach to capture this dependency.

To overcome these challenges, Time Sampler employs a combination of SSD (please refer to Sec.\ref{sec:stein_discrepancy} for details) and neural networks. SSD provides a powerful framework for fitting distributions with limited samples. It maintains a set of "particles" that are iteratively updated to approximate the target distribution; To capture the dependency of $q^*(t|\bm{x}_0,\bm{\epsilon}_\mu)$ on the input data, we use a neural network $\bm{T}_{\bm{\vartheta}}$ to parameterize the particle locations.

Denoting the empirical distribution $S_t=\{t_i|t_i\in[0,T]\}_{i=1}^n $, to approximate $q^*(t|\bm{x}_0,\bm{\epsilon}_\mu)$. Hence, we desin the Time Sampler $\bm{T}_{\bm{\vartheta}}\colon (\bm{x}_0, \bm{\epsilon}_\mu) \mapsto S_t$. The particles in Time Sampler are updated iteratively using a gradient-based approach.  The update direction is determined by minimizing the KL divergence between the empirical distribution of particles and the target distribution. 


Following \citep{liu2016stein}, let the updated distribution as $q\rightarrow q_{[\varepsilon {\phi}]}$, and we can use the direction of the fastest change in KL divergence as the update amount, recorded as ${\phi}^*$.
\begin{equation}
{\phi}^* = \arg\min_{{\phi} \in \mathcal{B}} \left\{ -\frac{d}{d\epsilon} \text{KL}(q_{[\epsilon {\phi}]} \| q^*) \bigg|_{\epsilon=0} \right\}
\end{equation}
where $\mathcal{B} = \{{\phi} \in \mathcal{H}: \|{\phi}\|_{\mathcal{H}} \le 1\}$. It can be proven that the direction of the fastest gradient is the vector function ${f}^* = {\beta}/\|{\beta}\|_{\mathcal{H}}$ (coms from property of ssd), that is,
\begin{equation}
-\frac{d}{d\varepsilon} \text{KL}(q_{[\varepsilon {\phi}]} \| q^*) \bigg|_{\varepsilon=0} = \mathbb{E}_{t \sim q}[\mathcal{A}_{q^*}^T {\phi}(t)]
\end{equation}
\begin{algorithm}[htbp]
    \caption{RayFlow Distillation Training}
    \label{alg:line_diffusion_training}
         \KwIn{Epochs $E$, Timesteps $T$,Images $S=\emptyset $
         Prompt datasets $S_{\bm{c}}=\{\bm{c}^{(i)}\}_{i=1}^N$,
         Time Sampler $\bm{T}_{\bm{\vartheta}}(\bm{x}_0, \bm{\epsilon}_\mu,t)$
         \\
         \textbf{Tunable parameter:} Network Parameters $\bm{\theta}$ and $\bm{\vartheta}$}
        \KwOut{Denoiser $\bm{\epsilon}_{\bm{\theta}}(\bm{\epsilon}, \bm{c}, t)$}
        \hrulefill \quad {Construct Distillation Data} \quad  \hrulefill \\
        \For{$i=1$ \KwTo $N$ }{
            Sampling $\hat{\bm{\epsilon}}^{(i)}$ from $\mathcal{N}(\bm{0}, \bm{I})$\\
            $(\hat{\bm{x}}_0^{(i)},\hat{\bm{\epsilon}}_\mu^{(i)}) = \bm{\Psi}(\bm{\epsilon}_{\bm{\theta}},\hat{\bm{\epsilon}}^{(i)}, \bm{c}^{(i)}, K)$\\$\min_{\bm{\vartheta}}\hat{\mathbb{E}}_k\|\bm{T}_{\bm{\vartheta}}(\hat{\bm{x}}^{(i)}_0, \hat{\bm{\epsilon}}^{(i)})_{[k]}-\xi_t(\hat{\bm{x}}^{(i)}_0, \hat{\bm{\epsilon}}^{(i)})\|_2^2$\\
            $S=S\cup  \{(\hat{\bm{x}}_0, \hat{\bm{\epsilon}}_\mu^{(i)})\}$\hfill $\triangleright$ Add data pair to dataset
        }
        \hrulefill \quad {RayFlow Training} \quad  \hrulefill \\
        \For{$e=1$ \KwTo $E$}{
        \For{$i=1$ \KwTo $N$}{
            $\{f_t|f_t=\bm{T}_{\bm{\vartheta}}(\hat{\bm{x}}^{(i)}_0, \hat{\bm{\epsilon}}_\mu^{(i)},t)\}_{t=1}^T$\hfill $\triangleright$ Time weight\\
            $p^*(t|\hat{\bm{x}}^{(i)}_0,\hat{\bm{\epsilon}}_\mu)=\frac{|f_t|}{\hat{\mathbb{E}}_{t}[|f_t|]}$\hfill $\triangleright$ Proposition.\ref{prop:RrayFlow}\\
            Sampling $t\sim p^*(t|\hat{\bm{x}}^{(i)}_0, \hat{\bm{\epsilon}_\mu})$\\$\min_{\bm{\theta}}\|\bm{\epsilon}_{\bm{\theta}}(\sqrt{\hat{\alpha}}_t\hat{\bm{x}}_0^{(i)}+(1-\sqrt{\hat{\alpha}}_t\hat{\bm{\epsilon}}_\mu^{(i)}))-\hat{\bm{\epsilon}}_\mu^{(i)}\|_2^2$\\
            $\min_{\bm{\vartheta}}\phi(\bm{T}_{\bm{\vartheta}}(\hat{\bm{x}}^{(i)}_0, \hat{\bm{\epsilon}}_\mu^{(i)})_{[t]})^2$
        }}
        \Return Network Parameters $\bm{\theta}$ and $\bm{\vartheta}$
\end{algorithm}

Here $\mathcal{A}_{q^*}^T {\phi}(t) = [\nabla_{t} \ln q^*(t)]^T {\phi}(t) + \nabla_{t}^T {\phi}(t)$. The gradient direction is
\begin{equation}
\begin{aligned}
{\phi}^*(\cdot) \propto {\beta}(\cdot) = \mathbb{E}_{t \sim q}[\mathcal{A}_{q^*} k(t, \cdot)] =\\ \mathbb{E}_{t \sim q}[\nabla_{t} \ln q^*(t) k(t, \cdot) + \nabla_{t} k(t, \cdot)]
\end{aligned}
\end{equation}
The first term in the parentheses represents the driving term, which makes the particle tend to the target distribution, and the second term represents the diffusion term, which prevents the particles from getting too close.
\begin{equation}
\begin{aligned}
\nabla_{t} \ln q^*(t)=\frac{\partial}{\partial t} \ln\left(\frac{1}{n}\sum_{i=1}^n \xi_t(\bm{x}_0,\bm{\epsilon}_\mu)K(t_i, t)\right)
\end{aligned}
\end{equation}
Then we can update Time Sampler $\bm{T}_{\bm{\vartheta}}$
\begin{equation}
\begin{aligned}
\min_{\bm{\vartheta}} L=\frac{1}{n}\sum_{i=1}^n(\bm{T}_{\bm{\vartheta}}(\bm{x}_0, \bm{\epsilon}_\mu)_{[i]}-(t_i+\varepsilon\phi(t_i)))^2
\end{aligned}
\end{equation}
It's notable that the above loss function is equavilent to $\frac{1}{n}\sum_{i=1}^n \phi(t_i)^2$. It's easy to understand we need to minimize the update distance until convergence.

\begin{algorithm}[t]
    \caption{RayFlow Distillation Sampling}
    \label{alg:line_diffusion_sampling}
        \KwIn{Sampling steps $K$, Prompt $\bm{c}$}
        \KwOut{Sampling result $\hat{\bm{x}}_0$}
        Sampling $\hat{\bm{x}}_K$ from $\mathcal{N}(\bm{0}, \bm{I})$\\
        \For{$k=K$ \KwTo $1$ }{
            $\hat{\bm{\epsilon}}_k = \bm{\epsilon}_{\bm{\theta}}(\hat{\bm{x}}_k, \bm{c}, k)$\hfill $\triangleright$ noise prediction\\
            $\hat{\bm{x}}_{k-1}=\frac{1}{\alpha_t}\hat{\bm{x}}_k-\frac{1-\alpha_t}{\alpha_t}\hat{\bm{\epsilon}}_k + \tilde{\beta}_t \bm{\epsilon}$\hfill $\triangleright$ $\bm{\epsilon}\sim \mathcal{N}(\bm{0}, \bm{I})$
        }
        \Return $\hat{\bm{x}}_0$
\end{algorithm}
\begin{algorithm}[t]
    \caption{RayFlow One-step Sampling}
    \label{alg:line_diffusion_one_step_sampling}
        \KwIn{Prompt $\bm{c}$}
        \KwOut{Sampling result $\hat{\bm{x}}_0$}
        Sampling $\hat{\bm{x}}_T$ from $\mathcal{N}(\bm{0}, \bm{I})$\\
        $\hat{\bm{\epsilon}}_T = \bm{\epsilon}_{\bm{\theta}}(\hat{\bm{x}}_T, \bm{c}, T)$\hfill $\triangleright$ noise prediction\\
        $\hat{\bm{x}}_{0}=\frac{1}{\sqrt{\bar{\alpha}_T}}\hat{\bm{x}}_T-\frac{1-\sqrt{\bar{\alpha}_T}}{\sqrt{\bar{\alpha}_T}}\hat{\bm{\epsilon}}_T + \tilde{\beta}_T \bm{\epsilon}$\hfill $\triangleright$ $\bm{\epsilon}\sim \mathcal{N}(\bm{0}, \bm{I})$\\
        \Return $\hat{\bm{x}}_0$
\end{algorithm}

\subsection{Algorithms}
We present the key algorithms that implement our RayFlow framework, consisting of the training procedure (Algo.\ref{alg:line_diffusion_training}) and two sampling approaches (Algo.\ref{alg:line_diffusion_sampling} and Algo.\ref{alg:line_diffusion_one_step_sampling}).

The training algorithm consists of two main phases: data construction and model training. In the data construction phase, we generate synthetic training pairs by sampling noise vectors. The Time Sampler $\bm{T}_{\bm{\vartheta}}$ learns to predict important timesteps for denoising. In the training phase, we optimize the denoising network $\bm{\epsilon}_{\bm{\theta}}$ to predict and remove noise at each timestep, guided by the learned time weights. $\bm{\Psi}(\bm{\epsilon}_{\bm{\theta}},\hat{\bm{\epsilon}}^{(i)}, \bm{c}^{(i)}, K)$ represents ODE solvers which sampling image from noise $\hat{\bm{\epsilon}}^{(i)}$ with condition $\bm{c}^{(i)}$ in $K$ steps, $\hat{\bm{\epsilon}}_\mu^{(i)}=\hat{\mathbb{E}}_k[\hat{\bm{\epsilon}}_k^{(i)}]$, where $\hat{\bm{\epsilon}}_k^{(i)}$ represents sampling result at timestep $k$.

we provide two sampling algorithms. The standard sampling  Algo.\ref{alg:line_diffusion_sampling} iteratively denoises the input over $K$ steps using the trained denoiser. Starting from pure noise $\hat{\bm{x}}_K$, it progressively refines the output through repeated noise prediction. For efficiency, we also present a one-step sampling variant Algo.\ref{alg:line_diffusion_one_step_sampling} that generates the final output directly using a single denoising step at timestep $T$.

\section{Experiment}
We present a series of experiments evaluating quality, scalability and robustness of RayFLow. Our training costs around 2.5 8 * A100 GPU days. We train LoRA instead of UNet for convenience. We employ \textit{AdamW} optimizer with learning rate $1e-6$, 16 batch size, and 200 epochs. We adopt gaussian kernel $K(x,x_i)=\exp\left(\frac{-\|x-x_i\|^2}{2h^2}\right)$ for Time Sampler, where $h$ is bandwidth, and set it $0.25$. For all of competitors, we follow their paper's setting.

\subsection{Implementation Details}
\paragraph{Dataset.} Our experiments utilize a carefully curated subset of the LAION \citep{schuhmann2022laion} and COYO datasets \citep{kakaobrain2022coyo-700m}, following the data selection approach outlined in previous methods \citep{lin2024sdxl, ren2024hyper}. The evaluation is performed on COCO-5k \citep{lin2014microsoft}, ImageNet \citep{deng2009imagenet},  Cifar 10 \citep{krizhevsky2009learning}, Cifar 100 \citep{krizhevsky2009learning} datasets.

\paragraph{Evaluation Metrics.} We employ multiple complementary metrics to assess the quality and performance. The aesthetic predictor, pre-trained on the LAION dataset, evaluates visual appeal, while CLIP score (ViT-B/32) measures text-to-image alignment. Moreover, we incorporate recent metrics including Image Reward \citep{xu2024imagereward} and Pick Score \cite{kirstain2023pick}.

\paragraph{Base Models.} Our experimental framework is built upon three existing popular models: stable-diffusion-v1-5 (SD15) \citep{rombach2022high} with UNet architecture, stable-diffusion-xl-v1.0-base (SDXL) \citep{rombach2022high} with UNet architecture, and PixArt \citep{chen2023pixart} implementing the DiT architecture. These models serve as benchmarks for comparing various acceleration schemes, with detailed performance comparisons presented in Tab.\ref{tab:quantitative_variant}.

\begin{figure}
    \centering
    \includegraphics[width=0.99\linewidth]{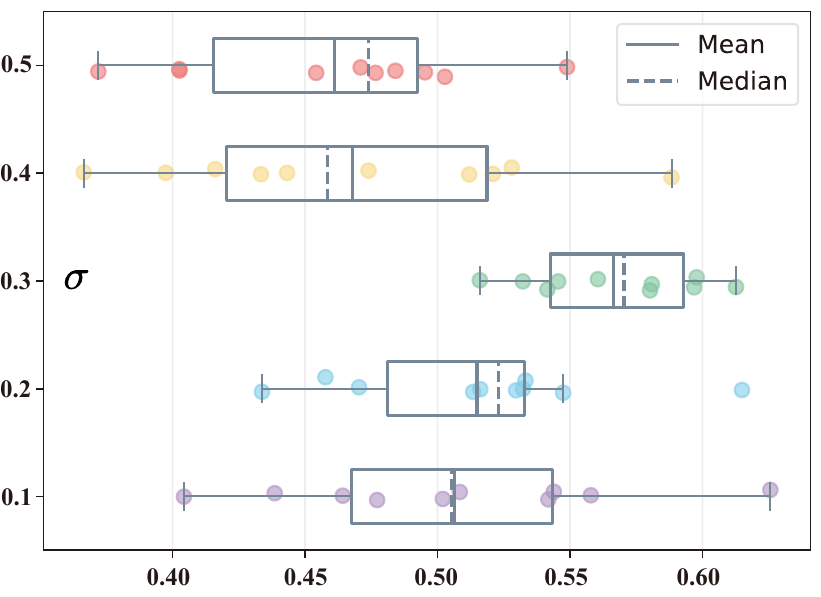}
    \vspace{-3mm}
    \caption{Sensitivity analysis on COCO-5K dataset where Aesthetis score for with respect to different values of $\sigma$.}
    \vspace{-1mm}\label{fig:sensitivity_analysis}
\end{figure}
\begin{figure}
    \centering
    
    \tablestyle{5pt}{1.0}
\setlength\tabcolsep{1.5pt}
\scalebox{1}{
    \begin{tabular}{l|ccc|ccc|ccc|ccc}
          &\multicolumn{3}{c|}{\textbf{COCO-5k}} & \multicolumn{3}{c|}{\textbf{ImageNet}} & \multicolumn{3}{c}{\textbf{Cifar-100}}&\multicolumn{3}{c}{\textbf{Cifar-10}} \\
    \textbf{Module} & Clip  & Aes & FID & Clip & Aes & FID & Clip & Aes & FID & Clip & Aes & FID \\
    \shline
    \rowcolor[rgb]{ .949,  .949,  .949} w/ $\bm{T}_{\bm{\vartheta}}$    &  \textbf{34.3}   & \textbf{5.9}  & \textbf{4.0}   & \textbf{36.0}   & \textbf{5.6}   & \textbf{1.9}   & \textbf{28.9}   & \textbf{4.9}   & \textbf{1.6} & \textbf{29.1}   & \textbf{4.9}   & \textbf{1.7} \\
    w/o $\bm{T}_{\bm{\vartheta}}$    & 31.9    & 5.7    & 4.8   & 33.9    & 5.0    & 2.8    & 26.4    & 4.7    & 3.5 & 28.0   & 4.7   & 3.9 \\
    \end{tabular}%
    } 
    \includegraphics[width=0.99\linewidth]{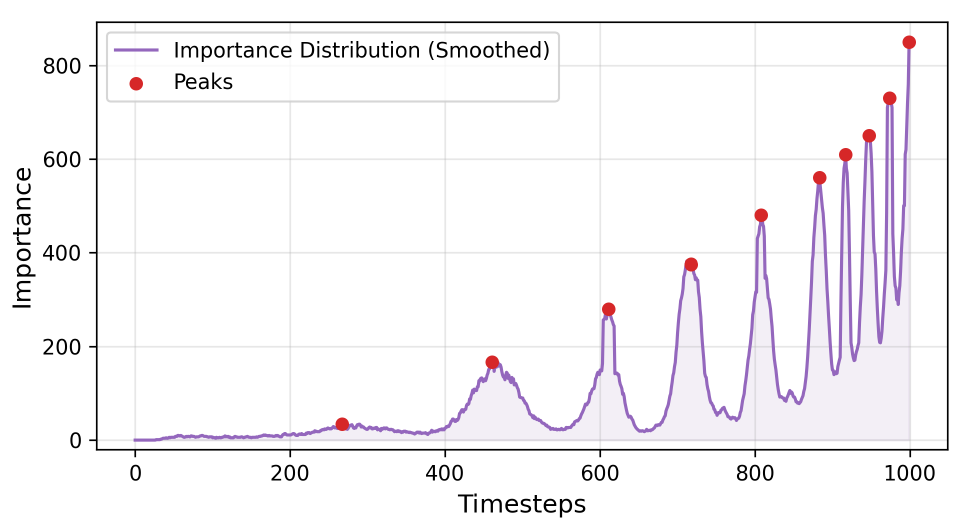}
    
    \vspace{-2mm}
  \caption{
  Visualization of the importance distribution of timestep. Performance of Time Sampler (SDXL) on different datasets.
  }
  \label{tab:timestep_importance}
  \vspace{-4mm}
\end{figure}

\begin{table*}[!t]
\setlength{\linewidth}{0.98\textwidth}
\tablestyle{5pt}{1.0}
\setlength\tabcolsep{4.5pt}  
\begin{tabular*}{\linewidth}{@{\extracolsep{\fill}}lcccc|cccc|cccc@{}}
    \multirow{2}[0]{*}{\textbf{Method}} & \multicolumn{4}{c|}{\textbf{Model Information}} & \multicolumn{4}{c|}{\textbf{FID-Score}} & \multicolumn{4}{c}{\textbf{Aesthetics-Score}} \\[-1.5pt]
          & Type     & Params     & Flow  & Distill Cost   & 1-Step     & 2-Step     & 4-Step     & 8-Step   & 1-Step     & 2-Step     & 4-Step     & 8-Step \\
    \shline
    \multicolumn{13}{c}{ \demph{ \textit{\textbf{Stable Diffusion} V1.5 Comparision} } }\\
    \hline
    SD15-Base \citep{rombach2022high}  & UNet  & 0.98B  & VP  & -  & 19.8\tiny{$\pm$.03}  & 12.1\tiny{$\pm$.06}  & 11.6\tiny{$\pm$.04}  & 9.02\tiny{$\pm$.02}   & 3.66\tiny{$\pm$.03}   & 3.89\tiny{$\pm$.05}  & 4.11\tiny{$\pm$.04}   & 4.54\tiny{$\pm$.01} \\
    SD15-PeRFlow \citep{yan2024perflow}  & LoRA  & 67.5M  &  RF & 6.4 Day  & 5.37\tiny{$\pm$.06} & 5.33\tiny{$\pm$.03} & 5.21\tiny{$\pm$.08} & 5.15\tiny{$\pm$.07} & 5.69\tiny{$\pm$.04} & 5.78\tiny{$\pm$.09} & 5.59\tiny{$\pm$.07} & 5.96\tiny{$\pm$.03} \\
    SD15-LCM \citep{luo2023latent} &LoRA  & 67.5M  & VP  & 4.5 Day & 5.29\tiny{$\pm$.05} & 5.05\tiny{$\pm$.07} & 5.22\tiny{$\pm$.04} & 5.03\tiny{$\pm$.09} & 5.71\tiny{$\pm$.06} & 5.83\tiny{$\pm$.07} & 5.86\tiny{$\pm$.05} & 5.96\tiny{$\pm$.04} \\
    SD15-TCD \citep{zheng2024trajectory} & LoRA  & 67.5M  & VP  &  5.5 Day  & 5.56\tiny{$\pm$.07} & 5.10\tiny{$\pm$.04} & 5.23\tiny{$\pm$.09} & 5.19\tiny{$\pm$.06} & 5.45\tiny{$\pm$.05} & 5.81\tiny{$\pm$.08} & 5.63\tiny{$\pm$.03} & 5.98\tiny{$\pm$.06}   \\
    Hyper-SD15 \citep{ren2024hyper} & LoRA  & 67.5M  & VP  &  12 Day & 5.41\tiny{$\pm$.04} & 5.05\tiny{$\pm$.06} & 5.21\tiny{$\pm$.05} & 5.02\tiny{$\pm$.04} & 5.62\tiny{$\pm$.03} & 5.85\tiny{$\pm$.04} & 5.64\tiny{$\pm$.05} & 6.00\tiny{$\pm$.05} \\
   \rowcolor[rgb]{ .949,  .949,  .949}\textbf{SD15-Ray (ours)} & LoRA  & 67.5M  & \textbf{Ray}  &  \textbf{2.5 Day} &  \textbf{5.10\tiny{$\pm$.03} } & \textbf{4.86\tiny{$\pm$.05} } & \textbf{4.78\tiny{$\pm$.04} } & \textbf{4.69\tiny{$\pm$.04} } & \textbf{5.92\tiny{$\pm$.02} } & \textbf{6.03\tiny{$\pm$.03} } & \textbf{6.13\tiny{$\pm$.04} } & \textbf{6.25\tiny{$\pm$.02} }\\
    \hline\\[-2.4ex]
    \multicolumn{13}{c}{ \demph{ \textit{\textbf{Stable Diffusion} XL Comparision} } }\\
    \hline
    SDXL-Base \citep{rombach2022high}  & UNet  & 3.5B  & VP  &  -&
    15.7\tiny{$\pm$.03}  & 10.5\tiny{$\pm$.02}  & 9.52\tiny{$\pm$.01}  & 8.42\tiny{$\pm$.04}   & 3.98\tiny{$\pm$.02}   & 4.14\tiny{$\pm$.05}   & 4.26\tiny{$\pm$.03}   & 4.61\tiny{$\pm$.02} \\
    SDXL-Turbo \citep{sauer2023adversarial}  & UNet  & 3.5B  & VP  &  N Day &4.32\tiny{$\pm$.07} & 4.16\tiny{$\pm$.05} & 4.00\tiny{$\pm$.04} & 3.80\tiny{$\pm$.06} & 5.75\tiny{$\pm$.08} & 5.85\tiny{$\pm$.04} & 5.90\tiny{$\pm$.09} & 5.70\tiny{$\pm$.05} \\
    SDXL-PeRFlow \citep{yan2024perflow}   & Unet  & 3.5B  & RF  &  6.4 Day &
    4.20\tiny{$\pm$.05} & 4.22\tiny{$\pm$.07} & 4.05\tiny{$\pm$.06} & 3.78\tiny{$\pm$.09} & 5.68\tiny{$\pm$.03} & 5.65\tiny{$\pm$.08} & 5.85\tiny{$\pm$.04} & 5.60\tiny{$\pm$.06} \\
    SDXL-LCM \citep{luo2023latent} & LoRA  & 197M  & VP  &  4.5 Day & 4.27\tiny{$\pm$.07} & 4.22\tiny{$\pm$.04} & 4.17\tiny{$\pm$.05} & 3.81\tiny{$\pm$.09} & 5.60\tiny{$\pm$.03} & 5.49\tiny{$\pm$.08} & 5.97\tiny{$\pm$.06} & 5.74\tiny{$\pm$.10} \\
    SDXL-TCD \citep{zheng2024trajectory} & LoRA  & 197M  & VP  &  5.5 Day&
    4.50\tiny{$\pm$.08} & 4.10\tiny{$\pm$.03} & 4.20\tiny{$\pm$.07} & 3.90\tiny{$\pm$.04} & 5.55\tiny{$\pm$.06} & 5.43\tiny{$\pm$.10} & 5.89\tiny{$\pm$.05} & 5.65\tiny{$\pm$.03} \\
    SDXL-Lightning \citep{lin2024sdxl} & LoRA  & 197M  & VP  &  N Day& 4.35\tiny{$\pm$.06} & 4.18\tiny{$\pm$.09} & 4.08\tiny{$\pm$.03} & 3.84\tiny{$\pm$.07} & 5.75\tiny{$\pm$.04} & 5.48\tiny{$\pm$.05} & 5.95\tiny{$\pm$.09} & 5.68\tiny{$\pm$.07} \\
    Hyper-SDXL \citep{ren2024hyper}& LoRA  & 197M  & VP  &  12 Day& 4.27\tiny{$\pm$.07} & 4.22\tiny{$\pm$.04} & 4.17\tiny{$\pm$.05} & 3.81\tiny{$\pm$.09} & 5.60\tiny{$\pm$.03} & 5.49\tiny{$\pm$.08} & 5.97\tiny{$\pm$.06} & 5.74\tiny{$\pm$.10} \\
    SDXL-DMD2 \citep{yin2024improved}& LoRA  & 197M  & VP  &  12 Day& 4.19\tiny{$\pm$.02} & 4.06\tiny{$\pm$.03} & 3.98\tiny{$\pm$.02} & 3.71\tiny{$\pm$.05} & 5.90\tiny{$\pm$.03} & 5.99\tiny{$\pm$.06} & 6.07\tiny{$\pm$.04} & 6.14\tiny{$\pm$.01} \\
    \rowcolor[rgb]{ .949,  .949,  .949} \textbf{SDXL-Ray (ours)} & LoRA  & 197M  & \textbf{Ray}  &  \textbf{2.5 Day} & \textbf{4.15}\tiny{$\pm$.03}  & \textbf{4.02}\tiny{$\pm$.02} & \textbf{3.90}\tiny{$\pm$.02} & \textbf{3.67}\tiny{$\pm$.04} & \textbf{5.96}\tiny{$\pm$.01} & \textbf{6.03}\tiny{$\pm$.02} & \textbf{6.15}\tiny{$\pm$.03} & \textbf{6.24}\tiny{$\pm$.01} \\
    \hline\\[-2.4ex]
    \multicolumn{13}{c}{ \demph{ \textit{\textbf{PixArt} DiT Comparision} } }\\
    \hline
    PixArt-$\alpha$ \citep{chen2023pixart}  & DiT  & 610M  & VP  &  - &
    13.6\tiny{$\pm$.02}  & 11.4\tiny{$\pm$.03}  & 10.14\tiny{$\pm$.01}  & 9.21\tiny{$\pm$.04}   & 3.78\tiny{$\pm$.02}   & 4.12\tiny{$\pm$.01}   & 4.42\tiny{$\pm$.03}   & 4.51\tiny{$\pm$.04} \\
    PixArt-$\Sigma$ \citep{chen2024pixart}  & DiT  & 610M  & VP  &  -  &12.9\tiny{$\pm$.02}  & 11.3\tiny{$\pm$.02}  & 10.04\tiny{$\pm$.03}  & 9.19\tiny{$\pm$.01}   & 3.84\tiny{$\pm$.04}   & 4.24\tiny{$\pm$.02}   & 4.56\tiny{$\pm$.01}   & 4.70\tiny{$\pm$.02} \\
    PixArt-$\delta$ \citep{chen2024pixart}   & DiT  & 610M  & VP  &  1 Day  &
    4.92\tiny{$\pm$.07}  & 4.76\tiny{$\pm$.05} & 4.47\tiny{$\pm$.09} & 4.10\tiny{$\pm$.04} & 5.54\tiny{$\pm$.06} & 5.70\tiny{$\pm$.08} & 5.85\tiny{$\pm$.03} & 5.95\tiny{$\pm$.09} \\
    \rowcolor[rgb]{ .949,  .949,  .949} \textbf{PixArt-Ray (ours)} & DiT  & 610M  & \textbf{Ray}  &  \textbf{1 Day} & \textbf{4.78}\tiny{$\pm$.02}  & \textbf{4.59}\tiny{$\pm$.03} & \textbf{4.12}\tiny{$\pm$.03} & \textbf{3.96}\tiny{$\pm$.01} & \textbf{5.84}\tiny{$\pm$.04} & \textbf{5.92}\tiny{$\pm$.02} & \textbf{6.10}\tiny{$\pm$.03} & \textbf{6.17}\tiny{$\pm$.03} \\
\end{tabular*}
\caption{
Quantitative comparison of state-of-the-arts models across various architectures and steps for FID and Aesthetic scores on the COCO-5k datasets. The best result is highlighted in \textbf{bold}. Distill Cost means 8 A100 days.}
\label{tab:quantitative_variant}
\vspace{-5mm}
\end{table*}

\subsection{Main Results}
\textbf{Timestep Importance Sampling.} We conduct some experiments to verify the performance of our proposed Time Sampler, shown in Tab.\ref{tab:timestep_importance}.
Model with Time Sampler outperforms its counterpart without the module across all metrics and datasets. We provide visualization of the timestep importance distribution, which supports the effectiveness of the module.  This visualization demonstrates how the module learns to adaptively focus on different timesteps during the diffusion process. This comprehensive evaluation across diverse datasets and metrics validates that the proposed Time Sampler module is indeed effective in enhancing the model's performance in image generation tasks.

\textbf{Sensitivity Analysis.}
For the variance coefficient $\sigma$ in RayFlow, we conduct experiment with different coefficients. We distill the SDXL model on the COCO-5k dataset using various $\sigma$. Fig.\ref{fig:sensitivity_analysis}  presents the sensitivity analysis results. We can observe that the model performs better when $\sigma$ is set to 0.3. When $\sigma$ is set to either 0.5 or 0.1, the results are relatively poor. This phenomenon might be explained by the fact that larger variance values can cover more samples and enhance diversity, while smaller variance values can reduce sampling instability. With moderate variance values, all of these beneficial characteristics are combined very well.

\textbf{Quantitative Comparison Between Acceleration Algorithms.} We conduct extensive experiments to validate the effectiveness of our proposed RayFlow framework, with results shown in Tab.\ref{tab:quantitative_variant}.Our experiments demonstrate consistent performance advantages across different model architectures. For SD15-based models, our SD15-Ray achieves the best performance across 1-8 sampling steps, outperforming both the baseline and other acceleration methods like LCM and TCD. In SDXL variants, while SDXL-PeRFlow shows superior single-step performance, our SDXL-Ray demonstrates better overall efficiency in multi-step scenarios, achieving FID scores for 2-8 steps. For PixArt architectures, PixArt-Ray maintains the leading position with the best FID and aesthetic scores. Notably, all our variants consistently show better performance, indicating that our acceleration approach maintains image quality while reducing computational requirements.

\begin{table}[h!]
    \centering
    \small

    \setlength\tabcolsep{4pt} 
   \begin{tabular}{m{2.5cm}m{0.6cm}m{0.4cm}m{0.7cm}m{0.6cm}m{0.8cm}m{0.75cm}}
      \toprule[1pt]

      \makebox[2.5cm][c]{\multirow{2}[-2]{*}{\textbf{Model}}}   & 
        \makebox[0.4cm][c]{\multirow{2}[-2]{*}{\textbf{Type}}}   & 
        \makebox[0.4cm][c]{\multirow{2}[-2]{*}{\textbf{FID}}} & 
      \makebox[0.7cm][c]{\makecell{\textbf{CLIP}\\\textbf{Score}}} & 
      \makebox[0.6cm][c]{\makecell{\textbf{Aes}\\\textbf{Score}}} & 
      \makebox[0.8cm][c]{\makecell{\textbf{Image}\\\textbf{Reward}}} & 
      \makebox[0.75cm][c]{\makecell{\textbf{Pick}\\\textbf{Score}}}\\
      
      \midrule[1pt]
      
      \makebox[2.5cm][c]{\textcolor{gray}{SD15-Base (25 step)}}  & 
      \makebox[0.6cm][c]{\textcolor{gray}{UNet}} & 
            \makebox[0.4cm][c]{\textcolor{gray}{5.08}} & 
      \makebox[0.7cm][c]{\textcolor{gray}{31.16}}    & 
      \makebox[0.6cm][c]{\textcolor{gray}{5.85}}   & 
      \makebox[0.8cm][c]{\textcolor{gray}{0.193}}  & 
      \makebox[0.75cm][c]{\textcolor{gray}{0.215}}  \\

\makebox[2.5cm][c]{SD15-ReFlow \citep{liu2022flow}} & 
      \makebox[0.6cm][c]{UNet} & 
      \makebox[0.4cm][c]{5.92} & 
      \makebox[0.7cm][c]{26.63} & 
      \makebox[0.6cm][c]{5.47} & 
      \makebox[0.8cm][c]{0.175} & 
      \makebox[0.75cm][c]{0.225} \\

       \makebox[2.5cm][c]{SD15-PeRFlow \citep{zheng2024trajectory}} & 
      \makebox[0.6cm][c]{UNet} & 
      \makebox[0.4cm][c]{5.56} & 
      \makebox[0.7cm][c]{28.77} & 
      \makebox[0.6cm][c]{5.61} & 
      \makebox[0.8cm][c]{0.137} & 
      \makebox[0.75cm][c]{0.225} \\

      \makebox[2.5cm][c]{SD15-ReDiff \citep{wang2024rectified}} & 
      \makebox[0.6cm][c]{UNet} & 
      \makebox[0.4cm][c]{5.69} & 
      \makebox[0.7cm][c]{30.31} & 
      \makebox[0.6cm][c]{5.41} & 
      \makebox[0.8cm][c]{0.184} & 
      \makebox[0.75cm][c]{0.204} \\
      
      \makebox[2.5cm][c]{SD15-InstaFlow \citep{liu2023instaflow}} & 
      \makebox[0.6cm][c]{LoRA} &
      \makebox[0.4cm][c]{6.10} & 
      \makebox[0.7cm][c]{29.45} & 
      \makebox[0.6cm][c]{5.76} & 
      \makebox[0.8cm][c]{0.128} & 
      \makebox[0.75cm][c]{0.223} \\

     \rowcolor{gray!10} \makebox[2.5cm][c]{\textbf{SD15-Ray (ours)}} & 
    \makebox[0.6cm][c]{LoRA}  & 
    \makebox[0.4cm][c]{\textbf{5.10}} & 
     \makebox[0.7cm][c]{\textbf{31.94}}   & 
     \makebox[0.6cm][c]{\textbf{5.92}} & 
     \makebox[0.8cm][c]{\textbf{0.205}} & 
     \makebox[0.75cm][c]{\textbf{0.228}} \\
  
      \midrule[0.5pt]
  
       \makebox[2.5cm][c]{\textcolor{gray}{SDXL-Base (25 step)}}  & 
       \makebox[0.6cm][c]{\textcolor{gray}{UNet}} & 
       \makebox[0.4cm][c]{\textcolor{gray}{4.28}} & 
       \makebox[0.7cm][c]{\textcolor{gray}{35.15}} & 
       \makebox[0.6cm][c]{\textcolor{gray}{5.86}}   & 
       \makebox[0.8cm][c]{\textcolor{gray}{0.905}}  & 
       \makebox[0.75cm][c]{\textcolor{gray}{0.222}}  \\

\makebox[2.5cm][c]{SDXL-ReFlow \citep{liu2022flow}} & 
       \makebox[0.6cm][c]{UNet} & 
       \makebox[0.4cm][c]{4.95} & 
       \makebox[0.7cm][c]{31.95} & 
       \makebox[0.6cm][c]{5.84} & 
       \makebox[0.8cm][c]{0.624} & 
       \makebox[0.75cm][c]{0.235} \\
       
        \makebox[2.5cm][c]{SDXL-PeRFlow \citep{zheng2024trajectory}} & 
       \makebox[0.6cm][c]{UNet} & 
       \makebox[0.4cm][c]{4.99} & 
       \makebox[0.7cm][c]{30.21} & 
       \makebox[0.6cm][c]{5.62} & 
       \makebox[0.8cm][c]{0.630} & 
       \makebox[0.75cm][c]{0.233} \\

        \makebox[2.5cm][c]{SDXL-ReDiff \citep{wang2024rectified}} & 
       \makebox[0.6cm][c]{UNet} & 
       \makebox[0.4cm][c]{4.51} & 
       \makebox[0.7cm][c]{34.02} & 
       \makebox[0.6cm][c]{5.24} & 
       \makebox[0.8cm][c]{0.785} & 
       \makebox[0.75cm][c]{0.220} \\
       
       \makebox[2.5cm][c]{SDXL-InstaFlow \citep{liu2023instaflow}} & 
       \makebox[0.6cm][c]{LoRA} & 
       \makebox[0.4cm][c]{4.69} & 
       \makebox[0.7cm][c]{30.81} & 
       \makebox[0.6cm][c]{5.57} & 
       \makebox[0.8cm][c]{0.442} & 
       \makebox[0.75cm][c]{0.240} \\
       
       \rowcolor{gray!10} \makebox[2.5cm][c]{\textbf{SDXL-Ray (ours)}} & 
       \makebox[0.6cm][c]{LoRA}  & 
       \makebox[0.4cm][c]{\textbf{4.15}}  & 
       \makebox[0.7cm][c]{\textbf{34.39}}  & 
       \makebox[0.6cm][c]{\textbf{5.96}} & 
       \makebox[0.8cm][c]{\textbf{1.021}} & 
       \makebox[0.75cm][c]{\textbf{0.251}} \\

      \bottomrule[0.5pt]
    \end{tabular}
    \vspace{-0.2cm}
        \caption{Quantitative comparisons on SD15, SDXL architectures between different noisy pair matching methods (two step).}\label{tab:quantitative_noise_match}
        \vspace{-6mm}
\end{table}

\begin{figure*}[t]
\centering
\includegraphics[width=1.0\textwidth]{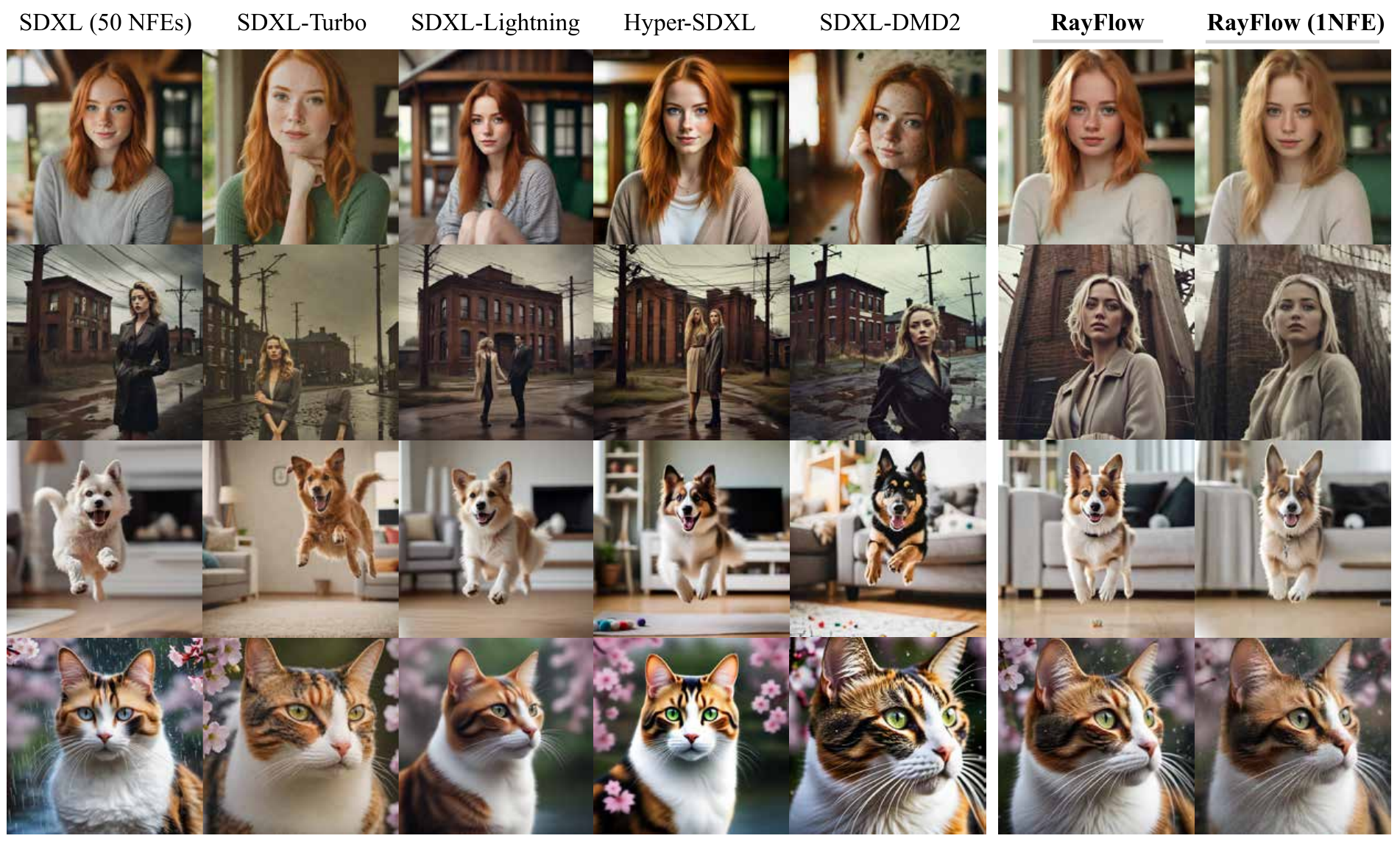}
\caption{Qualitative comparison of RayFlow against other few-step text-to-image models. Please zoom in to check details, lighting, and aesthetic performances. \textbf{All methods that do not have an NFE in place default to 4.}}
\label{fig:qualitative}
\vspace{-3mm}
\end{figure*}
\textbf{Quantitative Comparison Between Different Noise Matching Methods.} We  evaluate various noise matching methods derived from ReFlow \citep{liu2022flow} across SD15 and SDXL architectures, where shown in Tab.\ref{tab:quantitative_noise_match}. On both architectures, our method consistently outperforms existing approaches (ReFlow \citep{liu2022flow}, PeRFlow \cite{yan2024perflow}, RDiff \citep{rectified_flow}, and InstaFlow \citep{liu2023instaflow}) across multiple metrics. For SD15, RayFlow achieves the best score of most metrics, while for SDXL, it attains state-of-the-art performance. Notably, our LoRA-based implementation demonstrates superior performance in image reward, indicating better preservation of image quality during the distillation process compared to UNet-based approaches.


\textbf{Qualitative Comparison.}
Fig.\ref{fig:qualitative} presents a visual comparison between RayFlow and several state-of-the-art text-to-image models, including SDXL (50 NFEs), SDXL-Turbo, SDXL-Lightning, Hyper-SDXL, and DMD2. The comparison spans across diverse scenarios - portrait photography, urban landscapes, dynamic pet captures, and detailed animal close-ups. RayFlow demonstrates remarkable image generation capabilities, particularly in preserving fine details and maintaining visual coherence. Even in its single-NFE configuration, RayFlow produces images with comparable or superior quality to models requiring multiple steps. The results align with our quantitative findings in Tab.\ref{tab:quantitative_variant}, where RayFlow achieves exceptional aesthetic scores. These comprehensive evaluations establish RayFlow as a leading solution in high-efficiency text-to-image generation, effectively balancing quality with computational efficiency. Additional comparative results can be found in our supplementary materials.

\begin{table}
\tablestyle{5pt}{1.0}
\setlength\tabcolsep{3.5pt}

    \begin{tabular}{cc|cccc|cccc}
    \multicolumn{2}{c|}{\textbf{Module}} & \multicolumn{4}{c|}{\textbf{COCO-5k}} & \multicolumn{4}{c}{\textbf{ImageNet}} \\
    $\bm{T}_{\bm{\vartheta}}$  & $\bm{\epsilon}_\mu$ & Clip      & Aes &Img.R       & {FID}    & Clip      & Aes &Img.R       & {FID}  \\
    \shline
    \multicolumn{10}{c}{ \demph{ \it{Stable Diffusion V1.5 Ablation} } }\\
    \hline
    \rowcolor[rgb]{ .949,  .949,  .949} \cmark & \cmark & \textbf{31.94}  & \textbf{5.92}  & \textbf{0.205} & \textbf{5.10} & \textbf{34.24} & \textbf{5.47} & \textbf{0.214} & \textbf{2.42} \\    
    \cmark &   & 29.75 & 5.62 & 0.197 & 5.79 & 31.14 & 5.30 & 0.197 & 3.18 \\
          & \cmark & 30.96 & 5.66 & 0.198 & 5.89 & 31.96 & 5.00 & 0.196 & 3.29 \\
    \hline
    \multicolumn{10}{c}{ \demph{ \it{Stable Diffusion XL Ablation} } }\\
    \hline
    \rowcolor[rgb]{ .949,  .949,  .949} \cmark & \cmark  & \textbf{34.39}  & \textbf{5.96}  & \textbf{1.021} & \textbf{4.05} & \textbf{36.06} & \textbf{5.65} & \textbf{0.985} & \textbf{1.98} \\        
    \cmark &       & 32.67 & 5.48 & 0.980 & 4.77 & 32.77 & 5.31 & 0.897 & 2.92 \\
          & \cmark & 31.97 & 5.72 & 0.937 & 4.85 & 33.90 & 5.09 & 0.963 & 2.82 \\
    \hline
    \multicolumn{10}{c}{ \demph{ \it{PixArt DiT Ablation} } }\\
    \hline
    \rowcolor[rgb]{ .949,  .949,  .949} \cmark & \cmark  & \textbf{33.18}  & \textbf{5.84}  & \textbf{0.258} & \textbf{4.78} & \textbf{36.24} & \textbf{5.71} & \textbf{0.814} & \textbf{2.12} \\        
    \cmark &       & 31.23 & 5.61 & 0.230 & 5.46 & 34.43 & 5.25 & 0.792 & 2.93 \\
          & \cmark & 30.23 & 5.43 & 0.251 & 5.50 & 33.29 & 5.42 & 0.764 & 3.04 \\
    \end{tabular}%
   \vspace{-2mm}
  \caption{Performance (two steps) of ablation studies under different settings. $\bm{\epsilon}_\mu$ represents the RayFlow.
  }
  \label{tab:ablation-branch}%
  \vspace{-7mm}
\end{table}

\textbf{Ablation Studies.} We evaluate the effectiveness of two key components: the Time Sampler ($\bm{T}_{\bm{\vartheta}}$) and the RayFlow trajectory ($\bm{\epsilon}_\mu$). Experiments across SD15, SDXL, and PixArt architectures show that using $\bm{\epsilon}_\mu$ alone yields limited improvements, while $\bm{T}_{\bm{\vartheta}}$ alone achieves better but suboptimal results. The combination of both components consistently achieves the best performance, with significant improvements in evaluation metrics, demonstrating their complementary effects in enhancing generation quality.
\section{Conclusion}

In this work, we introduce RayFlow, a novel diffusion framework designed to address the inherent limitations of traditional diffusion models and existing acceleration techniques. By guiding samples along unique paths towards pre-computed target means, RayFlow ensures consistent expectations, minimizes path overlap, and improves sampling stability, all theoretically grounded in path probability maximization. Our Time Sampler, utilizing Stochastic Stein Discrepancies for importance sampling, further optimizes training efficiency by identifying crucial timesteps. Extensive experiments demonstrate RayFlow's superior performance in terms of sample quality, generation speed, and computational efficiency compared to existing acceleration methods. Future work includes extending RayFlow to other data modalities and exploring advanced strategies for target mean determination.
{
    \small
    \bibliographystyle{ieeenat_fullname}
    \bibliography{main}
}


\appendix
\onecolumn
\clearpage

\tableofcontents
\clearpage

\section{RayFlow}
\subsection{Flow Trajectory}\label{sec:RayFlow_trajectory}
\renewcommand{\theproposition}{\ref{prop:RrayFlow}}
\begin{proposition}
Given data $\bm{x}_0$, pretrained mean $\bm{\epsilon}_\mu \in \mathbb{R}^d$ and variance $\sigma \in \mathbb{R}_{>0}$, and target diffusion $p(\bm{x}_T) = \mathcal{N}(\bm{\epsilon}_\mu, \sigma^2\bm{I})$, we can describe the diffusion process with the following Markov chain.

\textbf{Flow Trajectories.} Define probability flow path.
\begin{equation}
    \psi_t(\cdot|\bm{\epsilon})=\sqrt{\bar{\alpha}_t}\bm{x}_0 + (1-\sqrt{\alpha}_t)\bm{\epsilon}_\mu + \sqrt{1-\bar{\alpha}_t}\bm{\epsilon}
\end{equation}

\textbf{Forward Process.} Add noise to image data.
\begin{equation*}
p(\bm{x}_t|\bm{x}_{t-1},\bm{\epsilon}_\mu)=\mathcal{N}\left(\alpha_t\bm{x}_{t-1}+(1-\alpha_t)\bm{\epsilon}_\mu, \beta_t^2\sigma^2\bm{I}\right)
\end{equation*}

\textbf{Backward Process.} Denoise from image data.
\begin{equation}
\begin{aligned}
&p(\bm{x}_{t-1}|\bm{x}_{t},\bm{\epsilon}_\mu)=\mathcal{N}\left(\frac{1}{\alpha_t} \bm{x}_t-\frac{1-\alpha_t}{\alpha_t}\bm{\epsilon}_\mu, \tilde{\beta}_t\sigma^2\bm{I}\right)
\end{aligned}  
\end{equation}
\begin{equation}
\tilde{\beta}_t=\left(\frac{(1-\alpha_t^2)(1-\bar{\alpha}_{t-1})}{1-\bar{\alpha}_{t}}\right)
\end{equation}  
\end{proposition}
\renewcommand{\theproposition}{\thesection.\arabic{proposition}}

\begin{proof}
We split the proof process into two parts: \textbf{Forward Process} and \textbf{Backward Process}.

\noindent\makebox[\linewidth]{\rule{\dimexpr0.5\linewidth-40pt}{0.8pt} \ \ \textbf{Forward Process} \ \ \rule{\dimexpr0.5\linewidth-40pt}{0.8pt}}
We set $\alpha_t^2+\beta_t^2=1$ and $\bar{\alpha}_t=\prod_{i=1}^t\alpha_t^2$, According to the forward process, we have:
\begin{equation}
\begin{aligned}
\bm{x}_t&=\alpha_t \bm{x}_{t-1} + (1-\alpha_t)\bm{\epsilon}_\mu + \beta_t \bm{\epsilon}_t\\
&=\alpha_t(\alpha_{t-1} \bm{x}_{t-2} + (1-\alpha_{t-1})\bm{\epsilon}_\mu + \beta_{t-1} \bm{\epsilon}_{t-1}) + (1-\alpha_t)\bm{\epsilon}_\mu+\beta_t\bm{\epsilon}_t\\
&=\alpha_t\alpha_{t-1}\cdots\alpha_1\bm{x}_0 \\
&\quad +\alpha_t\alpha_{t-1}\cdots\alpha_2(1-\alpha_{1})\bm{\epsilon}_\mu+\cdots+\alpha_t\alpha_{t-1}(1-\alpha_{t-2})\bm{\epsilon}_\mu + \alpha_t(1-\alpha_{t-1})\bm{\epsilon}_\mu + (1-\alpha_t)\bm{\epsilon}_\mu\\
&\quad + \alpha_t\alpha_{t-1}\cdots\alpha_2\beta_1\bm{\epsilon}_1+\alpha_t\alpha_{t-1}\cdots\alpha_3\beta_2\bm{\epsilon}_2+\cdots+\alpha_t\beta_{t-1}\bm{\epsilon}_t +\beta_t\bm{\epsilon}_t \\
&= \alpha_t\alpha_{t-1}\cdots\alpha_1\bm{x}_0 + (1-\alpha_t\alpha_{t-1}\cdots\alpha_1)\bm{\epsilon}_\mu \\
& \quad + \alpha_t\alpha_{t-1}\cdots\alpha_2\beta_1\bm{\epsilon}_1+\alpha_t\alpha_{t-1}\cdots\alpha_3\beta_2\bm{\epsilon}_2+\cdots+\alpha_t\beta_{t-1}\bm{\epsilon}_t +\beta_t\bm{\epsilon}_t 
\end{aligned}
\end{equation}
where
\begin{equation}
\begin{aligned}
\alpha_t\alpha_{t-1}\cdots\alpha_2\beta_1\bm{\epsilon}_1&\sim \mathcal{N}(\bm{0},\alpha_t^2\alpha_{t-1}^2\cdots\alpha_2^2\beta_1^2\sigma^2\bm{I} )\\
\alpha_t\alpha_{t-1}\cdots\alpha_3\beta_2\bm{\epsilon}_2&\sim \mathcal{N}(\bm{0},\alpha_t^2\alpha_{t-1}^2\cdots\alpha_3^2\beta_2^2\sigma^2\bm{I} )\\
&\cdots
\end{aligned}
\end{equation}
we set $\alpha_t^2 + \beta_t^2 = 1$, then we have:
\begin{equation}
\begin{aligned}
\alpha_t\alpha_{t-1}\cdots\alpha_2\beta_1\bm{\epsilon}_1&\sim \mathcal{N}(\bm{0},\alpha_t^2\alpha_{t-1}^2\cdots\alpha_2^2(1-\alpha_1^2)\sigma^2\bm{I} )\\
\alpha_t\alpha_{t-1}\cdots\alpha_3\beta_2\bm{\epsilon}_2&\sim \mathcal{N}(\bm{0},\alpha_t^2\alpha_{t-1}^2\cdots\alpha_3^2(1-\alpha_2^2)\sigma^2\bm{I} )\\
&\cdots
\end{aligned}
\end{equation}
based on the properties of Gaussian distributions:
\begin{equation}
\bm{a}\sim\mathcal{N}(\bm{\mu_1}, \sigma_1^2\bm{I}), \bm{b}\sim\mathcal{N}(\bm{\mu_2}, \sigma_2^2\bm{I}),
\end{equation}
\begin{equation}
\bm{a} + \bm{b} \sim \mathcal{N}(\bm{\mu_1} + \bm{\mu_2}, (\sigma_1^2 + \sigma_2^2)\bm{I}),
\end{equation}
then
\begin{equation}
\begin{aligned}
    &\quad \alpha_t^2\alpha_{t-1}^2\cdots\alpha_2^2(1-\alpha_1^2)\sigma^2 + \alpha_t^2\alpha_{t-1}^2\cdots\alpha_3^2(1-\alpha_2^2)\sigma^2\\
    &=\alpha_t^2\alpha_{t-1}^2\cdots\alpha_2^2\sigma^2 - \alpha_t^2\alpha_{t-1}^2\cdots\alpha_1^2\sigma^2+\alpha_t^2\alpha_{t-1}^2\cdots\alpha_3^2\sigma^2 - \alpha_t^2\alpha_{t-1}^2\cdots\alpha_2^2\sigma^2 \\
    &=\alpha_t^2\alpha_{t-1}^2\cdots\alpha_3^2\sigma^2 - \alpha_t^2\alpha_{t-1}^2\cdots\alpha_1^2\sigma^2.
\end{aligned}
\end{equation}
\begin{equation}
\begin{aligned}
&\alpha_t\alpha_{t-1}\cdots\alpha_2\beta_1\bm{\epsilon}_1+
\alpha_t\alpha_{t-1}\cdots\alpha_3\beta_2\bm{\epsilon}_2\\
&\sim \mathcal{N}(\bm{0}, (\alpha_t^2\alpha_{t-1}^2\cdots\alpha_3^2 - \alpha_t^2\alpha_{t-1}^2\cdots\alpha_1^2)\sigma^2\bm{I}).
\end{aligned}
\end{equation}

Let's expand it to $t$-dimension:
\begin{equation}
\begin{aligned}
& \alpha_t\alpha_{t-1}\cdots\alpha_2\beta_1\bm{\epsilon}_1+\alpha_t\alpha_{t-1}\cdots\alpha_3\beta_2\bm{\epsilon}_2\cdots+\beta_t\bm{\epsilon}_t=\bar{\bm{\epsilon}}_t\\
&\bar{\bm{\epsilon}}_t\sim\mathcal{N}\left(\bm{0}, (1-\alpha_t^2\alpha_{t-1}^2\cdots\alpha_1^2)\sigma^2\bm{I}\right),
\end{aligned}
\end{equation}
we have
\begin{equation}
\begin{aligned}
\bm{x}_t=\alpha_t\alpha_{t-1}\cdots\alpha_1\bm{x}_0+ (1-\alpha_t\alpha_{t-1}\cdots\alpha_1)\bm{\epsilon}_\mu + \sqrt{1-\alpha_t^2\alpha_{t-1}^2\cdots\alpha_1^2}\bar{\bm{\epsilon}}_t.
\end{aligned}
\end{equation}
Since, $\bar{\alpha}_t = \alpha_t^2\alpha_{t-1}^2\cdots\alpha_1^2$, then we can get:
\begin{equation}
p(\bm{x}_t|\bm{x}_0,\bm{\epsilon}_\mu)=\mathcal{N}\left(\sqrt{\bar{\alpha}_t}\bm{x}_0+(1-\sqrt{\bar{\alpha}_t})\bm{\epsilon}_\mu, (1-\bar{\alpha}_t)\sigma^2\bm{I}\right)
\end{equation}
Finally, we get:
\begin{equation}
\begin{aligned}
\text{(Noise Process)} \quad  & \quad p(\bm{x}_t|\bm{x}_{t-1},\bm{\epsilon}_\mu)=\mathcal{N}\left(\alpha_t\bm{x}_{t-1}+(1-\alpha_t)\bm{\epsilon}_\mu, \beta_t^2\sigma^2\bm{I}\right)\\
\text{(Noise Injection)} \quad  & \quad p(\bm{x}_t|\bm{x}_0,\bm{\epsilon}_\mu)=\mathcal{N}\left(\sqrt{\bar{\alpha}_t}\bm{x}_0+(1-\sqrt{\bar{\alpha}_t})\bm{\epsilon}_\mu, (1-\bar{\alpha}_t)\sigma^2\bm{I}\right)
\end{aligned}  
\end{equation}

\begin{equation}
\begin{aligned}
\text{(Noise Process)} \quad  & \quad \bm{x}_t=\alpha_t\bm{x}_{t-1}+ \bm{\epsilon}_t, \bm{\epsilon}_t \sim \mathcal{N}((1-\alpha_t)\bm{\epsilon}_\mu , \beta_t^2\sigma^2\bm{I})\\
\text{(Noise Injection)} \quad  &\quad \bm{x}_t=\sqrt{\bar{\alpha}_t}\bm{x}_0 + \bar{\bm{\epsilon}}_t, \bar{\bm{\epsilon}}_t\sim\mathcal{N}((1-\sqrt{\bar{\alpha}_t})\bm{\epsilon}_\mu, (1-\bar{\alpha}_t)\sigma^2\bm{I})
\end{aligned}  
\end{equation}
where $0<\alpha_t<1$ and $0<\beta_t<1$, $\sigma \in \mathbb{R}_{>0}$.

\noindent\makebox[\linewidth]{\rule{\dimexpr0.5\linewidth-40pt}{0.8pt} \ \ \textbf{Backward Process} \ \ \rule{\dimexpr0.5\linewidth-40pt}{0.8pt}}
During the backward process, we are focus on the reverse probability distribution $p(\bm{x}_{t-1}|\bm{x}_t, \bm{\epsilon}_\mu)$. Hence, we have:
\begin{equation}
\begin{aligned}
p(\bm{x}_{t-1}|\bm{x}_t, \bm{\epsilon}_\mu)
&=\frac{p(\bm{x}_{t}|\bm{x}_{t-1}, \bm{\epsilon}_\mu)p(\bm{x}_{t-1}|\bm{x}_0,\bm{\epsilon}_\mu)}{p(\bm{x}_{t}|\bm{x}_0,\bm{\epsilon}_\mu)}\\
&=\frac{\mathcal{N}\left(\alpha_t\bm{x}_{t-1}+(1-\alpha_t)\bm{\epsilon}_\mu, (1-\alpha_t^2)\sigma^2\bm{I}\right)\mathcal{N}\left(\sqrt{\bar{\alpha}_{t-1}}\bm{x}_0+(1-\sqrt{\bar{\alpha}_{t-1}})\bm{\epsilon}_\mu, (1-\bar{\alpha}_{t-1})\sigma^2\bm{I}\right)}{\mathcal{N}\left(\sqrt{\bar{\alpha}_t}\bm{x}_0+(1-\sqrt{\bar{\alpha}_t})\bm{\epsilon}_\mu, (1-\bar{\alpha}_t)\sigma^2\bm{I}\right)}\\
\end{aligned}
\end{equation}
Due to the fact that the multi-varite gaussian distribution has probability density function $p(\bm{x}) = \frac{1}{\sqrt{(2\pi)^k \det(\Sigma)}} \exp\left(-\frac{1}{2} (\bm{x} - \bm{\mu})^T \Sigma^{-1} (\bm{x} - \bm{\mu})\right)$, where $\bm{x}\in \mathbb{R}^k$ and $\bm{x}\sim\mathcal{N}(\bm{\mu}, \bm{\Sigma})$
\begin{equation}
\begin{aligned}
p(\bm{x}_{t-1}|\bm{x}_t, \bm{\epsilon}_\mu)&\propto \exp \left\{-\frac{1}{2\sigma^2}\left[\frac{(\bm{x}_t - \alpha_t\bm{x}_{t-1}-(1-\alpha_t)\bm{\epsilon}_\mu)^2}{(1-\alpha_t^2)}+\frac{(\bm{x}_{t-1} - \sqrt{\bar{\alpha}_{t-1}}\bm{x}_{0}-(1-\sqrt{\bar{\alpha}_{t-1}})\bm{\epsilon}_\mu)^2}{(1-\bar{\alpha}_{t-1})} \right. \right.\\
&\qquad \qquad \qquad \left. \left. - \frac{(\bm{x}_{t} - \sqrt{\bar{\alpha}_{t}}\bm{x}_{t}-(1-\sqrt{\bar{\alpha}_{t}})\bm{\epsilon}_\mu)^2}{(1-\bar{\alpha}_{t})} \right]\right\}
\end{aligned}
\end{equation}
Since $p(\bm{x}_{t-1}|\bm{x}_t, \bm{\epsilon}_\mu)$ is \wrt $\bm{x}_{t-1}$, we can convert this pdf to a simplified formulation (by rearanging terms which has no realtionship with $\bm{x}_{t-1}$ to $C(\bm{x}_t,\bm{x}_0, \bm{\epsilon}_\mu)$):
\begin{equation}
\begin{aligned}
p(\bm{x}_{t-1}|\bm{x}_t, \bm{\epsilon}_\mu)
&\propto\exp \left\{-\frac{1}{2\sigma^2}\left[\frac{\bm{x}_t^2+\alpha_t^2\bm{x}_{t-1}^2+(1-\alpha)^2\bm{\epsilon}_\mu^2-2\alpha_t\bm{x}_t\bm{x}_{t-1}-2(1-\alpha_t)\bm{x}_t\bm{\epsilon}_\mu+2\alpha_t(1-\alpha_t)\bm{x}_{t-1}\bm{\epsilon}_\mu}{1-\alpha_t^2}\right. \right.\\
&+ \left. \left.\frac{\bm{x}_{t-1}^2+\bar{\alpha}_{t-1}\bm{x}_{0}^2+(1-\sqrt{\bar{\alpha}_{t-1}})^2\bm{\epsilon}_\mu^2-2\sqrt{\bar{\alpha}_{t-1}}\bm{x}_0\bm{x}_{t-1}-(2-2\sqrt{\bar{\alpha}_{t-1}})\bm{x}_{t-1}\bm{\epsilon}_\mu+2(\sqrt{\bar{\alpha}_{t-1}}+\bar{\alpha}_{t-1})\bm{x}_{0}\bm{\epsilon}_\mu}{1-\bar{\alpha}_{t-1}} \right. \right.\\
&\qquad \qquad \qquad \left. \left. + C(\bm{x}_t,\bm{x}_0, \bm{\epsilon}_\mu)\right]\right\}\\
&=\exp \left\{-\frac{1}{2\sigma^2}\left[\frac{\alpha_t^2\bm{x}_{t-1}^2-2\alpha_t\bm{x}_t\bm{x}_{t-1}+2\alpha_t(1-\alpha_t)\bm{x}_{t-1}\bm{\epsilon}_\mu}{1-\alpha_t^2}+ \frac{\bm{x}_{t-1}^2-2\sqrt{\bar{\alpha}_{t-1}}\bm{x}_0\bm{x}_{t-1}-(2-2\sqrt{\bar{\alpha}_{t-1}})\bm{x}_{t-1}\bm{\epsilon}_\mu}{1-\bar{\alpha}_{t-1}} \right. \right.\\
&\qquad \qquad \qquad \left. \left. + C(\bm{x}_t,\bm{x}_0, \bm{\epsilon}_\mu)\right]\right\}
\end{aligned}
\end{equation}
Then we can rerange the remaining terms to construct a gaussian distribution formulation which associated with $\bm{x}_{t-1}$.
\begin{equation}
\begin{aligned}
&\quad p(\bm{x}_{t-1}|\bm{x}_t, \bm{\epsilon}_\mu)\\
&\propto\exp \left\{-\frac{1}{2\sigma^2}\left[\left(\frac{\alpha_t^2}{1-\alpha_t^2} + \frac{1}{1-\bar{\alpha}_{t-1}} \right)\bm{x}_{t-1}^2 -2\underbrace{\left(\frac{\alpha_t\bm{x}_t-\alpha_t(1-\alpha_t)\bm{\epsilon}_\mu}{1-\alpha_t^2}+\frac{\sqrt{\bar{\alpha}_{t-1}}\bm{x}_0+(1-\sqrt{\bar{\alpha}_{t-1}})\bm{\epsilon}_\mu}{1-\bar{\alpha}_{t-1}}\right)}_{(b)}\bm{x}_{t-1} \right. \right.\\
&\qquad \qquad \qquad \left. \left. + C(\bm{x}_t,\bm{x}_0, \bm{\epsilon}_\mu)\right]\right\}\\
&=\exp \left\{-\frac{1}{2\sigma^2}\left[\left(\frac{\alpha_t^2-\alpha_t^2\bar{\alpha}_{t-1}+1-\alpha_t^2}{(1-\alpha_t^2)(1-\bar{\alpha}_{t-1})}\right)\bm{x}_{t-1}^2 -2(b)\bm{x}_{t-1} + C(\bm{x}_t,\bm{x}_0, \bm{\epsilon}_\mu)\right]\right\}\\
&=\exp \left\{-\frac{1}{2\sigma^2}\left[\underbrace{\left(\frac{1-\bar{\alpha}_{t}}{(1-\alpha_t^2)(1-\bar{\alpha}_{t-1})}\right)}_{(a)}\bm{x}_{t-1}^2 -2(b)\bm{x}_{t-1} + C(\bm{x}_t,\bm{x}_0, \bm{\epsilon}_\mu)\right]\right\}\\
&=\exp \left\{-\frac{1}{2\sigma^2/(a)}\left[\bm{x}_{t-1}^2 -2\frac{(b)}{(a)}\bm{x}_{t-1} + C(\bm{x}_t,\bm{x}_0, \bm{\epsilon}_\mu)\right]\right\}\\
&=\exp \left\{-\frac{\left(\bm{x}_{t-1}-\frac{(b)}{(a)}\right)^2}{2\sigma^2(a)} + C(\bm{x}_t,\bm{x}_0, \bm{\epsilon}_\mu)\right\}\\
&\propto \mathcal{N}\left(\bm{\mu}_p(\bm{x}_t,\bm{x}_0,\bm{\epsilon}_\mu), \bm{\Sigma}_p(t)\bm{I}\right), \bm{\mu}_p(\bm{x}_t,\bm{x}_0,\bm{\epsilon}_\mu)=\frac{(b)}{(a)}, \bm{\Sigma}_p(t)=\sigma^2/(a)
\end{aligned}
\end{equation}
We can expand the $\bm{\mu}_p(\bm{x}_t,\bm{x}_0,\bm{\epsilon}_\mu)$

\begin{equation}\label{eq.mu_p}
\begin{aligned}
    & \quad \bm{\mu}_p(\bm{x}_t,\bm{x}_0,\bm{\epsilon}_\mu)\\
    &= \frac{(1-\alpha_t^2)(1-\bar{\alpha}_{t-1})}{1-\bar{\alpha}_{t}} \frac{(\alpha_t\bm{x}_t-\alpha_t(1-\alpha_t)\bm{\epsilon}_\mu)(1-\bar{\alpha}_{t-1})+(\sqrt{\bar{\alpha}_{t-1}}\bm{x}_0+(1-\sqrt{\bar{\alpha}_{t-1}})\bm{\epsilon}_\mu)(1-\alpha_t^2)}{(1-\alpha_t^2)(1-\bar{\alpha}_{t-1})}\\
    &=\frac{(\alpha_t\bm{x}_t-\alpha_t(1-\alpha_t)\bm{\epsilon}_\mu)(1-\bar{\alpha}_{t-1})+(\sqrt{\bar{\alpha}_{t-1}}\bm{x}_0+(1-\sqrt{\bar{\alpha}_{t-1}})\bm{\epsilon}_\mu)(1-\alpha_t^2)}{1-\bar{\alpha}_{t}}\\
    &=\frac{\alpha_t(1-\bar{\alpha}_{t-1})\bm{x}_t+\sqrt{\bar{\alpha}_{t-1}}(1-\alpha_t^2)\bm{x}_0+[\alpha_t^2-\alpha_t-\alpha_t^2\bar{\alpha}_{t-1}+\alpha_t\bar{\alpha}_{t-1}+(1-\alpha_t^2-\sqrt{\bar{\alpha}_{t-1}}+\sqrt{\bar{\alpha}_{t-1}}\alpha_t^2)]\bm{\epsilon}_\mu}{1-\bar{\alpha}_{t}}\\
    &=\frac{\alpha_t(1-\bar{\alpha}_{t-1})\bm{x}_t+\sqrt{\bar{\alpha}_{t-1}}(1-\alpha_t^2)\bm{x}_0+[1-\alpha_t-\alpha_t^2\bar{\alpha}_{t-1}+\alpha_t\bar{\alpha}_{t-1}-\sqrt{\bar{\alpha}_{t-1}}+\sqrt{\bar{\alpha}_{t-1}}\alpha_t^2]\bm{\epsilon}_\mu}{1-\bar{\alpha}_{t}}\\
    &=\frac{\alpha_t(1-\bar{\alpha}_{t-1})\bm{x}_t+\sqrt{\bar{\alpha}_{t-1}}(1-\alpha_t^2)\bm{x}_0}{1-\bar{\alpha}_{t}}+\underbrace{\frac{1-\alpha_t-\alpha_t^2\bar{\alpha}_{t-1}+\alpha_t\bar{\alpha}_{t-1}-\sqrt{\bar{\alpha}_{t-1}}+\sqrt{\bar{\alpha}_{t-1}}\alpha_t^2}{1-\bar{\alpha}_{t}}\bm{\epsilon}_\mu}_{(c)}\\
    &=\frac{\sqrt{\bar{\alpha}_{t-1}}(1-\alpha_t^2)}{1-\bar{\alpha}_{t}}\bm{x}_0 + \frac{\alpha_t(1-\bar{\alpha}_{t-1})}{1-\bar{\alpha}_{t}}\bm{x}_t+(c)\\
    &=\frac{\sqrt{\bar{\alpha}_{t-1}}(1-\alpha_t^2)}{1-\bar{\alpha}_{t}}\frac{\bm{x}_t-\mathbb{E}[\bar{\bm{\epsilon}}_t]}{\sqrt{\bar{\alpha}_t}}+ \frac{\alpha_t(1-\bar{\alpha}_{t-1})}{1-\bar{\alpha}_{t}}\bm{x}_t+(c), \bar{\bm{\epsilon}}_t\sim\mathcal{N}((1-\sqrt{\bar{\alpha}_t})\bm{\epsilon}_\mu, (1-\bar{\alpha}_t)\sigma^2\bm{I})\\
    &=\frac{\sqrt{\bar{\alpha}_{t-1}}(1-\alpha_t^2)}{1-\bar{\alpha}_{t}}\frac{\bm{x}_t}{\sqrt{\bar{\alpha}_t}}+ \frac{\alpha_t(1-\bar{\alpha}_{t-1})}{1-\bar{\alpha}_{t}}\bm{x}_t-\frac{\sqrt{\bar{\alpha}_{t-1}}(1-\alpha_t^2)}{(1-\bar{\alpha}_{t})\sqrt{\bar{\alpha}_t}}\mathbb{E}[\bar{\bm{\epsilon}}_t]+(c)\\
    &=\frac{\sqrt{\bar{\alpha}_{t-1}}(1-\alpha_t^2)}{1-\bar{\alpha}_{t}}\frac{\bm{x}_t}{\sqrt{\bar{\alpha}_t}}+ \frac{\alpha_t(1-\bar{\alpha}_{t-1})}{1-\bar{\alpha}_{t}}\bm{x}_t\underbrace{-\frac{\sqrt{\bar{\alpha}_{t-1}}(1-\alpha_t^2)(1-\sqrt{\bar{\alpha}_t})}{(1-\bar{\alpha}_{t})\sqrt{\bar{\alpha}_t}}\bm{\epsilon}_\mu}_{(d)}+(c)
\end{aligned}
\end{equation}
By reranging the terms, we can get:

\begin{equation}
\begin{aligned}
 \bm{\mu}_p(\bm{x}_t,\bm{x}_0,\bm{\epsilon}_\mu)
&= \left[\frac{\sqrt {{\bar{\alpha}_{t-1}}} (1-\alpha_t^2)}{(1-\bar{\alpha}_t) \sqrt{\bar{\alpha}_t}} + \frac{\alpha_t\left(1-\bar{\alpha}_{t-1}\right)}{1-\bar{\alpha}_t}\right]\bm{x}_t +(d) + (c) \\     
&= \left[\frac{1-\alpha_t^2}{(1-\bar{\alpha}_t) {\alpha}_t} + \frac{\alpha_t\left(1-\bar{\alpha}_{t-1}\right)}{1-\bar{\alpha}_t}\right]\bm{x}_t  +(d) + (c) \\     
&= \frac{(1-\alpha_t^2)+\alpha_t^2(1-\bar{\alpha}_{t-1})}{\left(1-\bar{\alpha}_t\right) \alpha_t} \bm{x}_t  +(d) + (c) \\     
&= \frac{1-\alpha_t^2+\alpha_t^2-\bar{\alpha}_{t}}{\left(1-\bar{\alpha}_t\right) \alpha_t} \bm{x}_t+(d) + (c)  \\     
&= \frac{1}{\alpha_t} \bm{x}_t +(d) + (c) 
\end{aligned}
\end{equation}

\begin{equation}\label{eq.simply_epsilon_mu}
\begin{aligned}
(c)+(d) &= \frac{ (1-\alpha_t-\alpha_t^2\bar{\alpha}_{t-1}+\alpha_t\bar{\alpha}_{t-1}-\sqrt{\bar{\alpha}_{t-1}}+\sqrt{\bar{\alpha}_{t-1}}\alpha_t^2)\sqrt{\bar{\alpha}_t} - \sqrt{\bar{\alpha}_{t-1}}(1-\alpha_t^2)(1-\sqrt{\bar{\alpha}_t})}{(1-\bar{\alpha}_t)\sqrt{\bar{\alpha}_t}}\bm{\epsilon}_\mu\\
&=\frac{(\sqrt{\bar{\alpha}_t}-\sqrt{\bar{\alpha}_t}\alpha_t-\sqrt{\bar{\alpha}_t}\alpha_t^2\bar{\alpha}_{t-1}+\sqrt{\bar{\alpha}_t}\alpha_t\bar{\alpha}_{t-1}-\sqrt{\bar{\alpha}_t}\sqrt{\bar{\alpha}_{t-1}}+\sqrt{\bar{\alpha}_t}\sqrt{\bar{\alpha}_{t-1}}\alpha_t^2)-\sqrt{\bar{\alpha}_{t-1}}(1-\alpha_t^2)(1-\sqrt{\bar{\alpha}_t})}{(1-\bar{\alpha}_t)\sqrt{\bar{\alpha}_t}}\bm{\epsilon}_\mu\\
&=\frac{(\sqrt{\bar{\alpha}_t}-\sqrt{\bar{\alpha}_t}\alpha_t-\sqrt{\bar{\alpha}_t}\alpha_t^2\bar{\alpha}_{t-1}+\sqrt{\bar{\alpha}_t}\alpha_t\bar{\alpha}_{t-1}-\sqrt{\bar{\alpha}_{t-1}} + \sqrt{\bar{\alpha}_{t-1}}\alpha_t^2)}{(1-\bar{\alpha}_t)\sqrt{\bar{\alpha}_t}}\bm{\epsilon}_\mu\\
&=\frac{(\sqrt{\bar{\alpha}_t}-\sqrt{\bar{\alpha}_t}\alpha_t^2\bar{\alpha}_{t-1}+\sqrt{\bar{\alpha}_t}\alpha_t\bar{\alpha}_{t-1}-\sqrt{\bar{\alpha}_{t-1}})}{(1-\bar{\alpha}_t)\sqrt{\bar{\alpha}_t}}\bm{\epsilon}_\mu\\
&=\frac{(\alpha_t\sqrt{\bar{\alpha}_{t-1}}-\sqrt{\bar{\alpha}_t}\alpha_t^2\bar{\alpha}_{t-1}+\sqrt{\bar{\alpha}_t}\alpha_t\bar{\alpha}_{t-1}-\sqrt{\bar{\alpha}_{t-1}})}{(1-\bar{\alpha}_t)\sqrt{\bar{\alpha}_t}}\bm{\epsilon}_\mu\\
&=\frac{(1-\alpha_t)(\sqrt{\bar{\alpha}_t}\alpha_t\bar{\alpha}_{t-1}-\sqrt{\bar{\alpha}_{t-1}})}{(1-\bar{\alpha}_t)\sqrt{\bar{\alpha}_t}}\bm{\epsilon}_\mu\\
&=\frac{(1-\alpha_t)(\sqrt{\bar{\alpha}_t}\bar{\alpha}_{t}/\alpha_t-\sqrt{\bar{\alpha}_{t}}/\alpha_t)}{(1-\bar{\alpha}_t)\sqrt{\bar{\alpha}_t}}\bm{\epsilon}_\mu\\
&=\frac{(1-\alpha_t)(\bar{\alpha}_{t}-1)}{(1-\bar{\alpha}_t)\alpha_t}\bm{\epsilon}_\mu\\
&=\frac{\alpha_t-1}{\alpha_t}\bm{\epsilon}_\mu
\end{aligned}
\end{equation}
Then we can get:
\begin{equation}
\bm{\mu}_p(\bm{x}_t,\bm{x}_0,\bm{\epsilon}_\mu) = \frac{1}{\alpha_t} \bm{x}_t - \frac{1-\alpha_t}{\alpha_t}\bm{\epsilon}_\mu
\end{equation}
Finally, we have
\begin{equation}
\begin{aligned}
\text{(Noise Process)} \quad \quad p(\bm{x}_{t-1}|\bm{x}_{t},\bm{\epsilon}_\mu)=\mathcal{N}\left(\frac{1}{\alpha_t} \bm{x}_t-\frac{1-\alpha_t}{\alpha_t}\bm{\epsilon}_\mu, \left(\frac{(1-\alpha_t^2)(1-\bar{\alpha}_{t-1})}{1-\bar{\alpha}_{t}}\right)\sigma^2\bm{I}\right)
\end{aligned}  
\end{equation}
\begin{equation}
\begin{aligned}
\text{(Noise Process)} \quad  \quad \bm{x}_{t-1}=\frac{1}{\alpha_t}\bm{x}_{t} - \bm{\epsilon}_t, \bm{\epsilon}_t \sim \mathcal{N}\left(\frac{1-\alpha_t}{\alpha_t}\bm{\epsilon}_\mu, \left(\frac{(1-\alpha_t^2)(1-\bar{\alpha}_{t-1})}{1-\bar{\alpha}_{t}}\right)\sigma^2\bm{I}\right)
\end{aligned}  
\end{equation}
\end{proof}

\subsection{Probability Path}\label{sec:probability_path}
\renewcommand{\theproposition}{\ref{prop:probability_path}}
\begin{proposition}
For any general diffusion which defined in Prop.\ref{prop:RrayFlow}, we have the path probability (the probability of starting from $\hat{\bm{x}}_0$ forward to $\hat{\bm{\epsilon}}_\mu$ and backward to $\hat{\bm{x}}_0$)
\begin{equation}
p(\hat{\bm{x}}_0\rightarrow\hat{\bm{\epsilon}}_\mu\rightarrow \hat{\bm{x}}_0) = p(\bm{x}_{T}=\hat{\bm{\epsilon}}_\mu|\bm{x}_0=\hat{\bm{x}}_0)p(\bm{x}_0=\hat{\bm{x}}_0|\bm{x}_{T}=\hat{\bm{\epsilon}}_\mu)
\end{equation}
where
\begin{equation}
p(\bm{x}_{0}|\bm{x}_T=\hat{\bm{\epsilon}}_\mu)=\mathcal{N}\left(\frac{\hat{\bm{\epsilon}}_\mu}{\sqrt{\bar{\alpha}_T}}  + \sum_{s=1}^{T}\frac{(e)_s+(c)_s}{\sqrt{\bar{\alpha}_{s-1}/\bar{\alpha}_{t-1}}}, \sum_{s=1}^T\frac{\tilde{\beta}_s\bar{\alpha}_t}{\bar{\alpha}_s}\sigma^2\bm{I}\right)
\end{equation}
and
\begin{equation}
(e)_s = -\frac{\sqrt{\bar{\alpha}_{s-1}}(1-\alpha_s^2)}{(1-\bar{\alpha}_{s})\sqrt{\bar{\alpha}_s}}\mathbb{E}[\bar{\bm{\epsilon}}_s]
\end{equation}
\begin{equation}
(c)_s = \frac{1-\alpha_s-\alpha_s^2\bar{\alpha}_{s-1}+\alpha_t\bar{\alpha}_{s-1}-\sqrt{\bar{\alpha}_{s-1}}+\sqrt{\bar{\alpha}_{s-1}}\alpha_s^2}{1-\bar{\alpha}_{s}}\bm{\epsilon}_\mu
\end{equation}
\end{proposition}
\renewcommand{\theproposition}{\thesection.\arabic{proposition}}

\begin{proof}
According to Eq.\eqref{eq.mu_p}, we have 
\begin{equation}
\begin{aligned}
\bm{\mu}_p(\bm{x}_t,\bm{x}_0,\bm{\epsilon}_\mu)&=\frac{\sqrt{\bar{\alpha}_{t-1}}(1-\alpha_t^2)}{1-\bar{\alpha}_{t}}\frac{\bm{x}_t-\mathbb{E}[\bar{\bm{\epsilon}}_t]}{\sqrt{\bar{\alpha}_t}}+ \frac{\alpha_t(1-\bar{\alpha}_{t-1})}{1-\bar{\alpha}_{t}}\bm{x}_t+(c)_t, \bar{\bm{\epsilon}}_t\sim\mathcal{N}((1-\sqrt{\bar{\alpha}_t})\bm{\epsilon}_\mu, (1-\bar{\alpha}_t)\sigma^2\bm{I})\\
&=\frac{\sqrt{\bar{\alpha}_{t-1}}(1-\alpha_t^2)}{(1-\bar{\alpha}_{t})\sqrt{\bar{\alpha}_t}}\bm{x}_t+ \frac{\alpha_t(1-\bar{\alpha}_{t-1})}{1-\bar{\alpha}_{t}}\bm{x}_t\underbrace{-\frac{\sqrt{\bar{\alpha}_{t-1}}(1-\alpha_t^2)\mathbb{E}[\bar{\bm{\epsilon}}_t]}{(1-\bar{\alpha}_{t})\sqrt{\bar{\alpha}_t}}}_{(e)_t}+(c)_t\\
&=\frac{1}{\alpha_t}\bm{x}_t + (e)_t + (c)_t,
\end{aligned}
\end{equation}
Then we can write the backward process probability:
\begin{equation}
p(\bm{x}_{t-1}|\bm{x}_{t},\bm{x}_T=\hat{\bm{\epsilon}}_\mu)=\mathcal{N}\left(\frac{1}{\alpha_t} \bm{x}_t+(e)_t+(c)_t, \tilde{\beta}_t\sigma^2\bm{I}\right)
\end{equation}
For the convenience of subsequent analysis, we start from timestep $T$
\begin{equation}
\begin{aligned}
p(\bm{x}_{T-1}|\bm{x}_T=\hat{\bm{\epsilon}}_\mu)&=\mathcal{N}\left(\frac{1}{\alpha_T} \hat{\bm{\epsilon}}_\mu+(e)_T+(c)_T, \tilde{\beta}_T\sigma^2\bm{I}\right)
\end{aligned}
\end{equation}
Then we can obtain the marginal distribution $p(\bm{x}_{T-2}|\bm{x}_T=\hat{\bm{\epsilon}}_\mu)$.
\begin{equation}
\begin{aligned}
p(\bm{x}_{T-2}|\bm{x}_T=\hat{\bm{\epsilon}}_\mu)&=\int_{\mathbb{R}^n}p(\bm{x}_{T-2}|\bm{x}_{T-1}) p(\bm{x}_{T-1}|\bm{x}_T=\hat{\bm{\epsilon}}_\mu) d\bm{x}_{T-1}\\
&=\int_{\mathbb{R}^n} \mathcal{N}\left(\frac{\bm{x}_{T-1}}{\alpha_{T-1}} + (e)_{T-1} + (c)_{T-1}, \tilde{\beta}_{T-1}\sigma^2\bm{I}\right)\mathcal{N}\left(\frac{1}{\alpha_T} \hat{\bm{\epsilon}}_\mu+ (e)_{T} + (c)_{T},\tilde{\beta}_{T}\sigma^2\bm{I}\right) d\bm{x}_{T-1}
\end{aligned}
\end{equation}
It's hard to calculate $p(\bm{x}_{T-2}|\bm{x}_T=\hat{\bm{\epsilon}}_\mu)$ directly. However, we can ensure that $p(\bm{x}_{T-2}|\bm{x}_T=\hat{\bm{\epsilon}}_\mu)$ is a gaussian distribution. Hence, we just need to calculate the paramter of gaussian distribution (mean and variance).
\begin{equation}
\begin{aligned}
&\quad \mathbb{E}_{\bm{x}_{T-2}\sim p(\bm{x}_{T-2}|\bm{x}_T=\hat{\bm{\epsilon}}_\mu)}[\bm{x}_{T-2}|\bm{x}_T=\hat{\bm{\epsilon}}_\mu] \\
&= \mathbb{E}_{\bm{x}_{T-2}\sim p(\bm{x}_{T-2}|\bm{x}_T=\hat{\bm{\epsilon}}_\mu)}[\mathbb{E}_{\bm{x}_{T-1}\sim p(\bm{x}_{T-1}|\bm{x}_T=\hat{\bm{\epsilon}}_\mu)}[\bm{x}_{T-2}|\bm{x}_{T-1}]|\bm{x}_T=\hat{\bm{\epsilon}}_\mu]\\
&= \mathbb{E}_{\bm{x}_{T-2}\sim p(\bm{x}_{T-2}|\bm{x}_T=\hat{\bm{\epsilon}}_\mu)}\left[\mathbb{E}_{\bm{x}_{T-1}\sim p(\bm{x}_{T-1}|\bm{x}_T=\hat{\bm{\epsilon}}_\mu)}\left[\frac{\bm{x}_{T-1}}{\alpha_{T-1}} +(e)_{T-1}+(c)_{T-1}|\bm{x}_{T-1}\right]|\bm{x}_T=\hat{\bm{\epsilon}}_\mu\right]\\
&= \mathbb{E}_{\bm{x}_{T-2}\sim p(\bm{x}_{T-2}|\bm{x}_T=\hat{\bm{\epsilon}}_\mu)}\left[\frac{1}{\alpha_{T-1}}\left(\frac{1}{\alpha_T} \hat{\bm{\epsilon}}_\mu+(e)_{T}+(c)_{T}\right)+(e)_{T-1} +(c)_{T-1}\right]\\
&=\frac{\hat{\bm{\epsilon}}_\mu}{\alpha_T\alpha_{T-1}} +\frac{(e)_{T}+(c)_{T}}{\alpha_{T-1}}+(e)_{T-1}+(c)_{T-1}
\end{aligned}
\end{equation}
Due to the fact that $\mathbb{V}[Y] = \mathbb{E}[\mathbb{V}[Y|X]] + \mathbb{V}[\mathbb{E}[Y|X]]$, hence, we can caculate the variance by:
\begin{equation}
\begin{aligned}
&\quad \mathbb{V}_{\bm{x}_{T-2}\sim p(\bm{x}_{T-2}|\bm{x}_T=\hat{\bm{\epsilon}}_\mu)}[\bm{x}_{T-2}|\bm{x}_T=\hat{\bm{\epsilon}}_\mu] \\
&= \mathbb{E}_{\bm{x}_{T-1} \sim p(\bm{x}_{T-1}|\bm{x}_T=\hat{\bm{\epsilon}}_\mu)}\left[\mathbb{V}_{\bm{x}_{T-2} \sim p(\bm{x}_{T-2}|\bm{x}_{T-1})}[\bm{x}_{T-2} | \bm{x}_{T-1}] | \bm{x}_T = \hat{\bm{\epsilon}}_\mu \right] \\
&\quad + \mathbb{V}_{\bm{x}_{T-1} \sim p(\bm{x}_{T-1}|\bm{x}_T=\hat{\bm{\epsilon}}_\mu)}\left[\mathbb{E}_{\bm{x}_{T-2} \sim p(\bm{x}_{T-2}|\bm{x}_{T-1})}[\bm{x}_{T-2} | \bm{x}_{T-1}] | \bm{x}_T = \hat{\bm{\epsilon}}_\mu \right]\\
&= \mathbb{E}_{\bm{x}_{T-1} \sim p(\bm{x}_{T-1}|\bm{x}_T=\hat{\bm{\epsilon}}_\mu)}\left[\tilde{\beta}_{T-1}\sigma^2\bm{I} | \bm{x}_T = \hat{\bm{\epsilon}}_\mu \right] \\
&\quad + \mathbb{V}_{\bm{x}_{T-1} \sim p(\bm{x}_{T-1}|\bm{x}_T=\hat{\bm{\epsilon}}_\mu)}\left[ \frac{1}{\alpha_{T-1}} \bm{x}_{T-1}-\frac{1-\alpha_{T-1}}{\alpha_{T-1}}\bm{\epsilon}_\mu | \bm{x}_T = \hat{\bm{\epsilon}}_\mu \right]\\
&= \tilde{\beta}_{T-1}\sigma^2\bm{I} + \left(\frac{1}{\alpha_{T-1}}\right)^2 \tilde{\beta}_T\sigma^2\bm{I}\\
&= \left( \tilde{\beta}_{T-1} + \frac{\tilde{\beta}_T}{\alpha_{T-1}^2} \right) \sigma^2\bm{I}
\end{aligned}
\end{equation}
where $\tilde{\beta}_t = \left(\frac{(1-\alpha_t^2)(1-\bar{\alpha}_{t-1})}{1-\bar{\alpha}_{t}}\right)$, After obtaining the paramter of $p(\bm{x}_{T-2}|\bm{x}_T=\hat{\bm{\epsilon}}_\mu)$, we can get
\begin{equation}
p(\bm{x}_{T-2}|\bm{x}_T=\hat{\bm{\epsilon}}_\mu)=\mathcal{N}\left(\frac{\hat{\bm{\epsilon}}_\mu}{\sqrt{\bar{\alpha}_T/\bar{\alpha}_{T-2}}}  +\frac{(e)_{T}+(c)_{T}}{\alpha_{T-1}}+(e)_{T-1}+(c)_{T-1}, \left( \tilde{\beta}_{T-1} + \frac{\tilde{\beta}_T}{\alpha_{T-1}^2} \right) \sigma^2\bm{I}\right)
\end{equation}
It's easy to get the general term $p(\bm{x}_{t-1}|\bm{x}_T=\hat{\bm{\epsilon}}_\mu)$ by mathematical induction.
\begin{equation}
\begin{aligned}
p(\bm{x}_{t-1}|\bm{x}_T=\hat{\bm{\epsilon}}_\mu)&=\mathcal{N}\left(\frac{\hat{\bm{\epsilon}}_\mu}{\sqrt{\bar{\alpha}_T/\bar{\alpha}_{t-1}}} + \sum_{s=t}^{T}\frac{(e)_s+(c)_s}{\sqrt{\bar{\alpha}_{s-1}/\bar{\alpha}_{t-1}}}, \sum_{s=t}^T\frac{\tilde{\beta}_s\bar{\alpha}_t}{\bar{\alpha}_s}\sigma^2\bm{I}\right)
\end{aligned}
\end{equation}
Finally, we can inference the probability density distribution at the timestep $0$.
\begin{equation}
\begin{aligned}
p(\bm{x}_{0}|\bm{x}_T=\hat{\bm{\epsilon}}_\mu)&=\mathcal{N}\left(\frac{\hat{\bm{\epsilon}}_\mu}{\sqrt{\bar{\alpha}_T}}  + \sum_{s=1}^{T}\frac{(e)_s+(c)_s}{\sqrt{\bar{\alpha}_{s-1}/\bar{\alpha}_{t-1}}}, \sum_{s=1}^T\frac{\tilde{\beta}_s\alpha_1}{\bar{\alpha}_s}\sigma^2\bm{I}\right)
\end{aligned}
\end{equation}
\end{proof}

\subsection{Optimal Probability Path}\label{sec:optimal_probability_path}
\renewcommand{\thetheorem}{\ref{thm:optimal_parameters}}
\begin{theorem}
Let $\bm{S}_{\bar{\bm{\epsilon}}}=\{\bar{\bm{\epsilon}}_t\}_{t=1}^T$ be the noise added in the forward process, $\bm{\epsilon}_\mu$, $\sigma$ be the parameters of the target distribution $\mathcal{N}(\bm{\epsilon}_\mu, \sigma^2\bm{I})$. For sample $\hat{\bm{x}}_0$, we can obtain the optimal parameters that maximize the path probability.
\begin{equation}
\begin{aligned}
\underset{\bm{S}_{\bar{\bm{\epsilon}}},\hat{\bm{\epsilon}}_\mu,\bm{\epsilon}_\mu,\sigma}{\arg\max}\ \ p(\bm{x}_{0}=\hat{\bm{x}}_0|\bm{x}_T=\hat{\bm{\epsilon}}_\mu)\prod_{t=1}^T p(\bar{\bm{\epsilon}}_t=\mathbb{E}[\bar{\bm{\epsilon}}_t])
\end{aligned}
\end{equation}
where the optimal parameters are defined by:
\begin{equation}
\begin{aligned}
    &\bm{S}_{\bar{\bm{\epsilon}}}^*=\{(1-\sqrt{\bar{\alpha}_t})\bm{\epsilon}_\mu\}_{t=1}^T, \ \sigma^* \rightarrow 0\\
    &\hat{\bm{\epsilon}}_\mu^* = \sqrt{\bar{\alpha}_T}\hat{\bm{x}}_0 + (1-\sqrt{\bar{\alpha}_T})\bm{\epsilon}_\mu,\ \bm{\epsilon}_\mu^*=\mathbb{E}_t[\mathbb{E}[\bar{\bm{\epsilon}_t}]]
\end{aligned}
\end{equation}
\end{theorem}
\renewcommand{\thetheorem}{\thesection.\arabic{theorem}}

\begin{proof}
According to Eq.\eqref{eq.simply_epsilon_mu}, we can rewrite the backward process probability distribution.
\begin{equation}
p(\bm{x}_{t-1}|\bm{x}_{t},\bm{x}_T=\hat{\bm{\epsilon}}_\mu^*)=\mathcal{N}\left(\frac{1}{\alpha_t} \bm{x}_t-\frac{1-\alpha_T}{\alpha_T}\bm{\epsilon}_\mu^*, \tilde{\beta}_t(\sigma^*)^2\bm{I}\right)
\end{equation}
For the convenience of subsequent analysis, we start from timestep $T$
\begin{equation}
\begin{aligned}
p(\bm{x}_{T-1}|\bm{x}_T=\hat{\bm{\epsilon}}_\mu^*)&=\mathcal{N}\left(\frac{1}{\alpha_T} \hat{\bm{\epsilon}}_\mu^*-\frac{1-\alpha_T}{\alpha_T}\bm{\epsilon}_\mu^*, \tilde{\beta}_T(\sigma^*)^2\bm{I}\right)
\end{aligned}
\end{equation}
Then we can obtain the marginal distribution $p(\bm{x}_{T-2}|\bm{x}_T=\hat{\bm{\epsilon}}_\mu^*)$.
\begin{equation}
\begin{aligned}
p(\bm{x}_{T-2}|\bm{x}_T=\hat{\bm{\epsilon}}_\mu^*)&=\int_{\mathbb{R}^n}p(\bm{x}_{T-2}|\bm{x}_{T-1}) p(\bm{x}_{T-1}|\bm{x}_T=\hat{\bm{\epsilon}}_\mu^*) d\bm{x}_{T-1}\\
&=\int_{\mathbb{R}^n} \mathcal{N}\left(\frac{\bm{x}_{T-1}}{\alpha_{T-1}} -\frac{1-\alpha_{T-1}}{\alpha_{T-1}}\bm{\epsilon}_\mu^*, \tilde{\beta}_{T-1}(\sigma^*)a^2\bm{I}\right)\mathcal{N}\left(\frac{1}{\alpha_T} \hat{\bm{\epsilon}}_\mu-\frac{1-\alpha_T}{\alpha_T}\bm{\epsilon}_\mu^*,\tilde{\beta}_{T}(\sigma^*)^2\bm{I}\right) d\bm{x}_{T-1}
\end{aligned}
\end{equation}
Similar to Prop.\ref{prop:probability_path}. It's hard to calculate $p(\bm{x}_{T-2}|\bm{x}_T=\hat{\bm{\epsilon}}_\mu^*)$ directly. However, we can ensure that $p(\bm{x}_{T-2}|\bm{x}_T=\hat{\bm{\epsilon}}_\mu^*)$ is a gaussian distribution. Hence, we just need to calculate the paramter of gaussian distribution (mean and variance).
\begin{equation}
\begin{aligned}
&\quad \mathbb{E}_{\bm{x}_{T-2}\sim p(\bm{x}_{T-2}|\bm{x}_T=\hat{\bm{\epsilon}}_\mu^*)}[\bm{x}_{T-2}|\bm{x}_T=\hat{\bm{\epsilon}}_\mu^*] \\
&= \mathbb{E}_{\bm{x}_{T-2}\sim p(\bm{x}_{T-2}|\bm{x}_T=\hat{\bm{\epsilon}}_\mu^*)}[\mathbb{E}_{\bm{x}_{T-1}\sim p(\bm{x}_{T-1}|\bm{x}_T=\hat{\bm{\epsilon}}_\mu^*)}[\bm{x}_{T-2}|\bm{x}_{T-1}]|\bm{x}_T=\hat{\bm{\epsilon}}_\mu^*]\\
&= \mathbb{E}_{\bm{x}_{T-2}\sim p(\bm{x}_{T-2}|\bm{x}_T=\hat{\bm{\epsilon}}_\mu^*)}\left[\mathbb{E}_{\bm{x}_{T-1}\sim p(\bm{x}_{T-1}|\bm{x}_T=\hat{\bm{\epsilon}}_\mu^*)}\left[\frac{\bm{x}_{T-1}}{\alpha_{T-1}} -\frac{1-\alpha_{T-1}}{\alpha_{T-1}}\bm{\epsilon}_\mu^*|\bm{x}_{T-1}\right]|\bm{x}_T=\hat{\bm{\epsilon}}_\mu^*\right]\\
&= \mathbb{E}_{\bm{x}_{T-2}\sim p(\bm{x}_{T-2}|\bm{x}_T=\hat{\bm{\epsilon}}_\mu^*)}\left[\frac{1}{\alpha_{T-1}}\left(\frac{1}{\alpha_T} \hat{\bm{\epsilon}}_\mu^*-\frac{1-\alpha_T}{\alpha_T}\bm{\epsilon}_\mu^*\right)-\frac{1-\alpha_{T-1}}{\alpha_{T-1}}\bm{\epsilon}_\mu^*|\bm{x}_{T-1}\right]\\
&=\frac{\hat{\bm{\epsilon}}_\mu^*}{\alpha_T\alpha_{T-1}} -\frac{1-\alpha_{T}\alpha_{T-1}}{\alpha_T\alpha_{T-1}}\bm{\epsilon}_\mu^*\\
&=\frac{\hat{\bm{\epsilon}}_\mu^*}{\sqrt{\bar{\alpha}_T/\bar{\alpha}_{T-2}}} -\frac{1-\sqrt{\bar{\alpha}_T/\bar{\alpha}_{T-2}}}{\sqrt{\bar{\alpha}_T/\bar{\alpha}_{T-2}}}\bm{\epsilon}_\mu^*
\end{aligned}
\end{equation}
and for the variance, we have
\begin{equation}
\begin{aligned}
&\quad \mathbb{V}_{\bm{x}_{T-2}\sim p(\bm{x}_{T-2}|\bm{x}_T=\hat{\bm{\epsilon}}_\mu^*)}[\bm{x}_{T-2}|\bm{x}_T=\hat{\bm{\epsilon}}_\mu^*] \\
&= \mathbb{E}_{\bm{x}_{T-1} \sim p(\bm{x}_{T-1}|\bm{x}_T=\hat{\bm{\epsilon}}_\mu^*)}\left[\mathbb{V}_{\bm{x}_{T-2} \sim p(\bm{x}_{T-2}|\bm{x}_{T-1})}[\bm{x}_{T-2} | \bm{x}_{T-1}] | \bm{x}_T = \hat{\bm{\epsilon}}_\mu^* \right] \\
&\quad + \mathbb{V}_{\bm{x}_{T-1} \sim p(\bm{x}_{T-1}|\bm{x}_T=\hat{\bm{\epsilon}}_\mu^*)}\left[\mathbb{E}_{\bm{x}_{T-2} \sim p(\bm{x}_{T-2}|\bm{x}_{T-1})}[\bm{x}_{T-2} | \bm{x}_{T-1}] | \bm{x}_T = \hat{\bm{\epsilon}}_\mu^* \right]\\
&= \mathbb{E}_{\bm{x}_{T-1} \sim p(\bm{x}_{T-1}|\bm{x}_T=\hat{\bm{\epsilon}}_\mu^*)}\left[\tilde{\beta}_{T-1}\sigma^2\bm{I} | \bm{x}_T = \hat{\bm{\epsilon}}_\mu^*\right] \\
&\quad + \mathbb{V}_{\bm{x}_{T-1} \sim p(\bm{x}_{T-1}|\bm{x}_T=\hat{\bm{\epsilon}}_\mu^*)}\left[ \frac{1}{\alpha_{T-1}} \bm{x}_{T-1}-\frac{1-\alpha_{T-1}}{\alpha_{T-1}}\bm{\epsilon}_\mu^* | \bm{x}_T = \hat{\bm{\epsilon}}_\mu^* \right]\\
&= \tilde{\beta}_{T-1}(\sigma^*)^2\bm{I} + \left(\frac{1}{\alpha_{T-1}}\right)^2 \tilde{\beta}_T(\sigma^*)^2\bm{I}\\
&= \left( \tilde{\beta}_{T-1} + \frac{\tilde{\beta}_T}{\alpha_{T-1}^2} \right) (\sigma^*)^2\bm{I}
\end{aligned}
\end{equation}
After obtaining the paramter of $p(\bm{x}_{T-2}|\bm{x}_T=\hat{\bm{\epsilon}}_\mu^*)$, we can get
\begin{equation}
p(\bm{x}_{T-2}|\bm{x}_T=\hat{\bm{\epsilon}}_\mu)=\mathcal{N}\left(\frac{\hat{\bm{\epsilon}}_\mu}{\sqrt{\bar{\alpha}_T/\bar{\alpha}_{T-2}}} -\frac{1-\sqrt{\bar{\alpha}_T/\bar{\alpha}_{T-2}}}{\sqrt{\bar{\alpha}_T/\bar{\alpha}_{T-2}}}\bm{\epsilon}_\mu, \left( \tilde{\beta}_{T-1} + \frac{\tilde{\beta}_T}{\alpha_{T-1}^2} \right) \sigma^2\bm{I}\right)
\end{equation}
It's easy to get the general term $p(\bm{x}_{t-1}|\bm{x}_T=\hat{\bm{\epsilon}}_\mu)$ by mathematical induction.
\begin{equation}
\begin{aligned}
p(\bm{x}_{t-1}|\bm{x}_T=\hat{\bm{\epsilon}}_\mu)&=\mathcal{N}\left(\frac{\hat{\bm{\epsilon}}_\mu}{\sqrt{\bar{\alpha}_T/\bar{\alpha}_{t-1}}} -\frac{1-\sqrt{\bar{\alpha}_T/\bar{\alpha}_{t-1}}}{\sqrt{\bar{\alpha}_T/\bar{\alpha}_{t-1}}}\bm{\epsilon}_\mu, \sum_{s=t}^T\frac{\tilde{\beta}_s\bar{\alpha}_t}{\bar{\alpha}_s}\sigma^2\bm{I}\right)
\end{aligned}
\end{equation}
Finally, we can inference the probability density distribution at the timestep $0$.
\begin{equation}
\begin{aligned}
p(\bm{x}_{0}|\bm{x}_T=\hat{\bm{\epsilon}}_\mu^*)&=\mathcal{N}\left(\frac{\hat{\bm{\epsilon}}_\mu^*}{\sqrt{\bar{\alpha}_T}} -\frac{1-\sqrt{\bar{\alpha}_T}}{\sqrt{\bar{\alpha}_T}}\bm{\epsilon}_\mu^*, \sum_{s=t}^T\frac{\tilde{\beta}_s\bar{\alpha}_t}{\bar{\alpha}_s}(\sigma^*)^2\bm{I}\right)\\
&=\mathcal{N}\left(\frac{\sqrt{\bar{\alpha}_T}\hat{\bm{x}}_0 + (1-\sqrt{\bar{\alpha}_T})\bm{\epsilon}_\mu^*}{\sqrt{\bar{\alpha}_T}} -\frac{1-\sqrt{\bar{\alpha}_T}}{\sqrt{\bar{\alpha}_T}}\bm{\epsilon}_\mu^*, \sum_{s=t}^T\frac{\tilde{\beta}_s\bar{\alpha}_t}{\bar{\alpha}_s}(\sigma^*)^2\bm{I}\right)\\
&=\mathcal{N}\left(\hat{\bm{x}}_0, \sum_{s=t}^T\frac{\tilde{\beta}_s\bar{\alpha}_t}{\bar{\alpha}_s}(\sigma^*)^2\bm{I}\right)=\mathcal{N}\left(\bm{0},\bm{I}\right)
\end{aligned}
\end{equation}
Hence, we can induce
\begin{equation}
p(\bm{x}_{0}=\hat{\bm{x}}_0|\bm{x}_T=\hat{\bm{\epsilon}}_\mu^*) = 1\geq p(\bm{x}_{0}=\hat{\bm{x}}_0|\bm{x}_T=\hat{\bm{\epsilon}}_\mu)
\end{equation}
$p(\bm{x}_{0}=\hat{\bm{x}}_0|\bm{x}_T=\hat{\bm{\epsilon}}_\mu^*)$ is the optimal probability path, and $\bm{S}_{\bar{\bm{\epsilon}}}^*,\hat{\bm{\epsilon}}_\mu^*,\bm{\epsilon}_{\mu}^*, \sigma^*$ are optimal parameters.
\end{proof}

\subsection{Optimal Denoiser}
Let us assume that out training set consists of a finite number of samples $\{\bm{x}_0^{(1)}, \cdots, \bm{x}_0^{(K)}\}$ and target mean set $\{\bm{\epsilon}^{(1)}_\mu,\cdots,\bm{\epsilon}^{(K)}_\mu\}$. This implies $p(\bm{x}_0)$ is representsed by a mixture of Dirac delta distributions:
\begin{equation}
p(\bm{x}_0) = \frac{1}{K}\sum_{i=1}^K\delta(\bm{x}_0-\bm{x}_0^{(i)})
\end{equation}
Let us now consider the denoising loss. By expanding the expectations, we can rewrite the formula as an integral over the noisy samples $\bm{x}$
\begin{equation}
\begin{aligned}
&\quad \mathbb{E}_{\bm{x}_0^{(i)}\sim p(\bm{x}_0)}[\mathbb{E}_{\bm{x}_t\sim p(\bm{x}_t|\bm{x}_0)}[\|\bm{\epsilon}_\theta(\bm{x}_t)-\bm{\epsilon}_\mu^{(i)}\|_2^2]]\\
&=\mathbb{E}_{\bm{x}_0^{(i)}\sim p(\bm{x}_0)}\left[\int_{\mathbb{R}^n}\mathcal{N}(\sqrt{\bar{\alpha}_t}\bm{x}_0^{(i)}+(1-\sqrt{\bar{\alpha}_t})\bm{\epsilon}_\mu^{(i)}, (1-\bar{\alpha}_t)\bm{I})\|\bm{\epsilon}_\theta(\bm{x}_t)-\bm{\epsilon}_\mu^{(i)}\|_2^2 \ d\bm{x}_t\right]\\
&=\frac{1}{K}\sum_{i=1}^K \int_{\mathbb{R}^n}\mathcal{N}(\sqrt{\bar{\alpha}_t}\bm{x}_0^{(i)}+(1-\sqrt{\bar{\alpha}_t})\bm{\epsilon}_\mu^{(i)}, (1-\bar{\alpha}_t)\bm{I})\|\bm{\epsilon}_\theta(\bm{x}_t)-\bm{\epsilon}_\mu^{(i)}\|_2^2 \ d\bm{x}_t\\
&=\int_{\mathbb{R}^n} \underbrace{\frac{1}{K}\sum_{i=1}^K\mathcal{N}(\sqrt{\bar{\alpha}_t}\bm{x}_0^{(i)}+(1-\sqrt{\bar{\alpha}_t})\bm{\epsilon}_\mu^{(i)}, (1-\bar{\alpha}_t)\bm{I})\|\bm{\epsilon}_\theta(\bm{x}_t)-\bm{\epsilon}_\mu^{(i)}\|_2^2}_{\mathcal{L}(\bm{\epsilon}_\theta)} \ d\bm{x}_t
\end{aligned}
\end{equation}
We can minimize $\mathcal{L}(\bm{\epsilon}_\theta)$ independently for each $\bm{x}_t$:
\begin{equation}
\bm{\epsilon}_\theta^* = \arg\min_{\bm{\epsilon}_\theta}\mathcal{L}(\bm{\epsilon}_\theta) 
\end{equation}
This is a convex optimization problem; its solution is uniquely identified by setting the gradient \wrt $\bm{\epsilon}_\theta$ to zero:
\begin{equation}
\begin{aligned}
&\quad \nabla_{\bm{\epsilon}_\theta} \mathcal{L}(\bm{\epsilon}_\theta)\\
&= \nabla_{\bm{\epsilon}_\theta} \left[\frac{1}{K}\sum_{i=1}^K\mathcal{N}(\sqrt{\bar{\alpha}_t}\bm{x}_0^{(i)}+(1-\sqrt{\bar{\alpha}_t})\bm{\epsilon}_\mu^{(i)}, (1-\bar{\alpha}_t)\bm{I})\|\bm{\epsilon}_\theta(\bm{x}_t)-\bm{\epsilon}_\mu^{(i)}\|_2^2\right]\\
&=\left[\frac{1}{K}\sum_{i=1}^K\mathcal{N}(\sqrt{\bar{\alpha}_t}\bm{x}_0^{(i)}+(1-\sqrt{\bar{\alpha}_t})\bm{\epsilon}_\mu^{(i)}, (1-\bar{\alpha}_t)\bm{I})\right]\nabla_{\bm{\epsilon}_\theta} \|\bm{\epsilon}_\theta(\bm{x}_t)-\bm{\epsilon}_\mu^{(i)}\|_2^2\\
&=\left[\frac{1}{K}\sum_{i=1}^K\mathcal{N}(\sqrt{\bar{\alpha}_t}\bm{x}_0^{(i)}+(1-\sqrt{\bar{\alpha}_t})\bm{\epsilon}_\mu^{(i)}, (1-\bar{\alpha}_t)\bm{I})\right](2\bm{\epsilon}_\theta(\bm{x}_t)-2\bm{\epsilon}_\mu^{(i)})\\
&=\left[\frac{1}{K}\sum_{i=1}^K\mathcal{N}(\sqrt{\bar{\alpha}_t}\bm{x}_0^{(i)}+(1-\sqrt{\bar{\alpha}_t})\bm{\epsilon}_\mu^{(i)}, (1-\bar{\alpha}_t)\bm{I})\right]\bm{\epsilon}_\theta(\bm{x}_t)-\left[\frac{1}{K}\sum_{i=1}^K\mathcal{N}(\sqrt{\bar{\alpha}_t}\bm{x}_0^{(i)}+(1-\sqrt{\bar{\alpha}_t})\bm{\epsilon}_\mu^{(i)}, (1-\bar{\alpha}_t)\bm{I})\right]\bm{\epsilon}_\mu^{(i)}
\end{aligned}
\end{equation}
Then we can solve the optimal denoiser by letting $\nabla_{\bm{\epsilon}_\theta}\mathcal{L}(\bm{\epsilon}_\theta) =0 $
\begin{equation}
    \bm{\epsilon}_\theta^*(\bm{x}_t)=\frac{\frac{1}{K}\sum_{i=1}^K\mathcal{N}(\sqrt{\bar{\alpha}_t}\bm{x}_0^{(i)}+(1-\sqrt{\bar{\alpha}_t})\bm{\epsilon}_\mu^{(i)}, (1-\bar{\alpha}_t)\bm{I})\bm{\epsilon}_\mu^{(i)}}{\frac{1}{K}\sum_{i=1}^K\mathcal{N}(\sqrt{\bar{\alpha}_t}\bm{x}_0^{(i)}+(1-\sqrt{\bar{\alpha}_t})\bm{\epsilon}_\mu^{(i)}, (1-\bar{\alpha}_t)\bm{I})}
\end{equation}

\section{Timestep Importance Sampling}
\subsection{Stein Identity}\label{sec:setin_identity}
The KSD is thus given by the norm of \(\bm{\beta}\) in \(\mathcal{H}^d\):


\begin{lemma}
For any bounded and smooth function $f(x)$, we have $\mathbb{E}_p[s_q(x)f(x)+\nabla_x f(x)]=0$, where $s_q(x) = \nabla_x \ln q(x)$ is the score function of distribution $q$.
\end{lemma}

\begin{proof}
Expanding the expectation:
\begin{equation}\label{eq.2}
    \int s_q(x)f(x)p(x) dx + \int \nabla_x f(x) p(x) dx = 0
\end{equation}

Let's focus on the first term using integration by parts:
\begin{equation}
\begin{aligned}
    \int s_q(x)f(x)p(x)dx &= \int\nabla_x\ln q(x) f(x) p(x) dx
\end{aligned}
\end{equation}

Let $u=f(x)$ and $dv=\nabla_x\ln q(x)p(x)dx$. Then $du=\nabla_x f(x) dx$ and $v=\int \nabla_x \ln q(x)p(x)dx$.

Applying integration by parts:
\begin{equation}
\begin{aligned}
\int s_q(x)f(x)p(x)dx &= \left[f(x)\int \nabla_x \ln q(x)p(x)dx\right]^\infty_{-\infty} - \int \nabla_x f(x) \left(\int \nabla_x \ln q(x)p(x)dx\right) dx
\end{aligned}
\end{equation}

Let $A = \int \nabla_x \ln q(x)p(x)dx$. Substituting back into Eq.\eqref{eq.2}:
\begin{equation}
\begin{aligned}
    \int s_q(x)f(x)p(x) dx + \int \nabla_x f(x) p(x) dx = \left[f(x)A\right]^\infty_{-\infty}-\int \nabla_x f(x)(p(x)-A)dx
\end{aligned}
\end{equation}

Now, let's examine the two terms on the right-hand side:

1. $\left[f(x)A\right]^\infty_{-\infty}$: 
   This term vanishes because $f(x)$ is bounded and $q(x)$ approaches zero at infinity, implying $A$ also approaches zero at infinity.

2. $\int \nabla_x f(x)(p(x)-A)dx$:
   This term equals zero if and only if $p(x) = A$ for all $x$.

Since $A$ is a constant, we have:
\begin{equation}
\nabla_x p(x) = \nabla_x A = 0
\end{equation}

This implies:
\begin{equation}
\frac{\nabla_x p(x)}{p(x)} = \nabla_x \ln q(x)
\end{equation}

Which is true if and only if $p(x) = q(x)$, as $\nabla_x \ln p(x) = \frac{\nabla_x p(x)}{p(x)}$.

Therefore, when $p(x) = q(x)$, both terms on the right-hand side are zero, proving the theorem.
\end{proof}

\subsection{Stein Discrepancy} \label{sec:stein_discrepancy}

Let's develop the concept of Stein Discrepancy as a measure of distance between probability distributions.

\subsubsection{Basic Definition}
Consider two smooth distributions $p(\mathbf{x})$ and $q(\mathbf{x})$ on $\mathbb{R}^d$. The Stein score function for distribution $q$ is defined as:
\begin{equation}
s_q(\mathbf{x}) = \nabla_{\mathbf{x}} \ln q(\mathbf{x})
\end{equation}

A fundamental property (proved in Sec.\ref{sec:setin_identity}) states that:
\begin{equation}
\mathbb{E}p[s_q(\mathbf{x})f(\mathbf{x}) + \nabla{\mathbf{x}}f(\mathbf{x})] = 0
\end{equation}
holds if and only if $p(\mathbf{x}) = q(\mathbf{x})$, where $f$ is a smooth function.

\subsubsection{Stein Discrepancy Measure}
Based on this property, we can define a discrepancy measure between distributions:
\begin{equation}
S(p,q) = \max_{f \in \mathcal{F}} (\mathbb{E}p[s_q(\mathbf{x})f(\mathbf{x}) + \nabla{\mathbf{x}}f(\mathbf{x})])^2
\end{equation}
where $\mathcal{F}$ is the set of smooth functions. This measure is positive when $p \neq q$, but is generally difficult to compute directly.

\subsubsection{RKHS Framework}
To make the discrepancy measure computationally tractable, we can utilize Reproducing Kernel Hilbert Space (RKHS) theory. Let's recall key RKHS concepts:

For a positive definite kernel $k(\mathbf{x}, \mathbf{y})$, Mercer's theorem gives its spectral decomposition:
\begin{equation}
k(\mathbf{x}, \mathbf{y}) = \sum_j \lambda_j \mathbf{e}_j(\mathbf{x}) \mathbf{e}_j(\mathbf{y})^T
\end{equation}
where ${\mathbf{e}_j}$ are orthogonal basis functions and ${\lambda_j}$ are eigenvalues.
The RKHS $\mathcal{H}$ generated by kernel $k(\cdot, \cdot): \mathcal{X} \times \mathcal{X} \rightarrow \mathbb{R}$ has two key properties:
$k(\mathbf{x}, \cdot) \in \mathcal{H}$ for any $\mathbf{x} \in \mathcal{X}$
Reproducing property: $f(\mathbf{x}) = \langle f, k(\mathbf{x}, \cdot) \rangle_{\mathcal{H}}$ for any $f \in \mathcal{H}$
We extend this to $\mathcal{H}^d = \mathcal{H} \times \cdots \times \mathcal{H}$ ($d$ times) for vector-valued functions with inner product:
\begin{equation}
\langle \mathbf{f}, \mathbf{g} \rangle_{\mathcal{H}^d} = \sum_{i=1}^d \langle f_i, g_i \rangle_{\mathcal{H}}
\end{equation}
\subsubsection{Kernelized Stein Discrepancy (KSD)}
The Stein operator $\mathcal{A}p$ for distribution $p$ is defined as:
\begin{equation}
\mathcal{A}p f(\mathbf{x}) = s_p(\mathbf{x})\cdot f(\mathbf{x}) + \nabla{\mathbf{x}}\cdot f(\mathbf{x}) = \frac{1}{p(\mathbf{x})} \nabla{\mathbf{x}}\cdot [p(\mathbf{x})f(\mathbf{x})]
\end{equation}

This operator satisfies:
\begin{equation}
\int_{\mathbf{x} \in \mathcal{X}} \nabla_{\mathbf{x}}\cdot (f(\mathbf{x})p(\mathbf{x}))d\mathbf{x} = 0
\end{equation}

A key relationship for KSD is:
\begin{equation}
\mathbb{E}_p[\mathcal{A}_q\mathbf{f}(\mathbf{x})] = \mathbb{E}_p[\mathcal{A}_q\mathbf{f}(\mathbf{x})-\mathcal{A}_p\mathbf{f}(\mathbf{x})] = \mathbb{E}_p[(s_q(\mathbf{x})-s_p(\mathbf{x}))\mathbf{f}(\mathbf{x})^T]
\end{equation}

This formulation provides a computationally tractable way to measure discrepancy between distributions using kernel methods.

\subsection{Optimal sampling Distribution}\label{sec:optimal_sampling_distribution}
\renewcommand{\theproposition}{\ref{prop:optimal_sampling_distribution}}
\begin{proposition}
The optimal sampling distribution for Eq.\eqref{Eq.diffusion_train} with minimal variance is:
\begin{equation}
\begin{aligned}
q^*(t|\bm{x}_0,\bm{\epsilon}_\mu)\propto\xi_t(\bm{x}_0,\bm{\epsilon}_\mu)p(t),
\end{aligned}
\end{equation}
where $\xi_t(\bm{x}_0,\bm{\epsilon}_\mu)=\|\bm{\epsilon}_{\bm{\theta}}(\sqrt{\bar{\alpha}_t}\hat{\bm{x}_0}+(1-\sqrt{\bar{\alpha}_t})\bm{\epsilon}_\mu)-\bm{\epsilon}_\mu\|_2^2$.
which means for any probability  distribution $p$, we have
\begin{equation*}
\begin{aligned}
\mathbb{V}_{t\sim q^*(t),(\bm{x}_0,\bm{\epsilon}_\mu)}[\xi_t(\bm{x}_0,\bm{\epsilon}_\mu)]\leq\mathbb{V}_{t\sim p(t),(\bm{x}_0,\bm{\epsilon}_\mu)}[\xi_t(\bm{x}_0,\bm{\epsilon}_\mu)]
\end{aligned}
\end{equation*}
\end{proposition}
\renewcommand{\thetheorem}{\thesection.\arabic{theorem}}

\begin{proof}
We aim to minimize the variance:
\[
\mathbb{V}_{t \sim q(t), (\bm{x}_0, \bm{\epsilon}_\mu)} \left[ \xi_t(\bm{x}_0, \bm{\epsilon}_\mu) \right]
\]

This can be expressed as:

\[
\int \frac{\xi_t(\bm{x}_0, \bm{\epsilon}_\mu) p(t)^2}{q(t)} \, dt
\]

subject to \( \int q(t) \, dt = 1 \). The Lagrangian is:

\[
\int \frac{\xi_t(\bm{x}_0, \bm{\epsilon}_\mu) p(t)^2}{q(t)} \, dt + \lambda \left( \int q(t) \, dt - 1 \right)
\]

Taking the derivative with respect to \( q(t) \) and setting it to zero gives:

\[
-\frac{\xi_t(\bm{x}_0, \bm{\epsilon}_\mu) p(t)^2}{q(t)^2} + \lambda = 0
\]

Solving for \( q(t) \), we get:

\[
q^*(t) = \sqrt{\frac{\xi_t(\bm{x}_0, \bm{\epsilon}_\mu) p(t)^2}{\lambda}}
\]

This implies:

\[
q^*(t) \propto \xi_t(\bm{x}_0, \bm{\epsilon}_\mu) p(t)
\]

Using the Cauchy-Schwarz inequality, we show:

\[
\sigma_{q^*}^2 \leq \sigma_{p}^2
\]

Thus, the distribution \( q^*(t|\bm{x}_0,\bm{\epsilon}_\mu) \) minimizes the variance for the given problem, proving the proposition.

\end{proof}

\end{document}